\documentclass[11pt]{article}

\usepackage{chaotao}

\usepackage[preprint]{neurips_2019}

\author{%
  Chao Tao\\
  Computer Science Department\\
  Indiana University at Bloomington
  \And
  Sa\'{u}l Blanco\\
  Computer Science Department\\
  Indiana University at Bloomington
  \And
  Jian Peng\\
  Computer Science Department\\
  University of Illinois at Urbana-Champaign
  \And
  Yuan Zhou\\
  Computer Science Department, Indiana University at Bloomington\\
  Department of ISE, University of Illinois at Urbana-Champaign
}
  
\usepackage[utf8]{inputenc} 
\usepackage[T1]{fontenc}    
\usepackage{hyperref}       
\usepackage{url}            
\usepackage{booktabs}       
\usepackage{amsfonts}       
\usepackage{nicefrac}       
\usepackage{microtype}      

\newcommand{\LSA}{{\rm LSA}\xspace}
\newcommand{\KL}{D_{\rm{KL}}}
\newcommand{\stat}{\xi}
\newcommand{\num}{f(\varkappa)}

\title{Thresholding Bandit with Optimal Aggregate Regret}

\begin{document}

\maketitle

\begin{abstract}

We consider the thresholding bandit problem,  whose goal is to find arms of mean rewards above a given threshold $\theta$, with a fixed budget of $T$ trials. We introduce LSA, a new, simple and anytime algorithm that aims to minimize the aggregate regret (or the expected number of mis-classified arms). We prove that our algorithm is instance-wise asymptotically optimal. We also provide comprehensive empirical results to demonstrate the algorithm's superior performance over existing algorithms under a variety of different scenarios. 
\end{abstract}

\allowdisplaybreaks

\vspace{-2ex}

\section{Introduction}
\vspace{-2ex}

The stochastic \emph{Multi-Armed Bandit} (MAB) problem has been extensively studied in the past decade ~\cite{auer2002using, AudibertBM10, BubeckMS09,GabillonGL12, KarninKS13, jamieson2014lil, garivier2016optimal, chen2017towards}. In the classical framework, at each trial of the game, a learner faces a set of $K$ \emph{arms}, pulls an arm and receives an unknown stochastic reward. Of particular interest is the \emph{fixed budget} setting, in which the learner is only given a limited number of total pulls. Based on the received rewards, the learner will recommend the best arm, i.e., the arm with the highest mean reward. In this paper, we study a variant of the MAB problem, called the \emph{Thresholding Bandit Problem} (TBP). In TBP, instead of finding the best arm, we expect the learner to identify all the arms whose mean rewards $\theta_i$ ($i \in \{1, 2, 3, \dots, K\}$) are greater than or equal to a given threshold $\theta$. This is a very natural setting with direct real-world applications to active binary classification and anomaly detection~\cite{locatelli2016optimal, steinwart2005classification}.

In this paper, we propose to study TBP under the notion of \emph{aggregate regret}, which is defined as the expected number of errors after $T$ samples of the bandit game. Specifically, for a given algorithm $\mathbb{A}$ and a TBP instance $I$ with $K$ arms, if we use $e_i$ to denote the probability that the algorithm makes an incorrect decision corresponding to arm $i$ after $T$ rounds of samples, the {aggregate regret} is defined as $\mathcal{R}^{\mathbb{A}}(I; T) \eqdef \sum_{i=1}^{K} e_i$. In contrast, most previous works on TBP aim to minimize the \emph{simple regret}, which is the probability that at least one of the arms is incorrectly labeled. Note that the definition of aggregate regret directly reflects the overall classification accuracy of the TBP algorithm, which is more meaningful than the simple regret in many real-world applications. For example, in the crowdsourced binary labeling problem, the learner faces $K$ binary classification tasks, where each task $i$ is associated with a latent true label $z_i \in \{0, 1\}$, and a latent soft-label $\theta_i$. The soft-label $\theta_i$ may be used to model the \emph{labeling difficulty/ambiguity} of the task, in the sense that $\theta_i$ fraction of the crowd will label task $i$ as $1$ and the rest labels task $i$ as $0$. The crowd is also assumed to be \emph{reliable}, i.e., $z_i = 0$ if and only if $\theta_i \geq \frac{1}{2}$. The goal of the crowdsourcing problem is to sequentially query a random worker from the large crowd about his/her label on task $i$ for a budget of $T$ times, and then label the tasks with as high (expected) accuracy as possible. If we treat each of the binary classification task as a Bernoulli arm with mean reward $\theta_i$, then this crowdsourced problem becomes aggregate regret minimization in TBP with $\theta = \frac{1}{2}$. If a few tasks are extremely ambiguous (i.e., $\theta_i \to \frac{1}{2}$), the simple regret would trivially approach $1$ (i.e., every algorithm would almost always fail to correctly label all tasks). In such cases, however, a good learner may turn to accurately label the less ambiguous tasks and still achieve a meaningful aggregate regret.

A new challenge arising for the TBP with aggregate regret is how to balance the exploration for each arm given a fixed budget. Different from the exploration vs.\ exploitation trade-off in the classical MAB problems, where exploration is only aimed for finding the best arm, the TBP expects to maximize the accuracy of the classification of \emph{all} arms. Let $\Delta_i \eqdef |\theta_i - \theta|$ be the \emph{hardness} parameter or \emph{gap} for each arm $i$. An arm with smaller $\Delta_i$ would need more samples to achieve the same classification confidence. A TBP learner faces the following dilemma -- whether to allocate samples to determine the classification of one hard arm, or use it for improving the accuracy of another easier arm. 

\vspace{-1ex}
\paragraph{Related Work.} 
Since we focus on the TBP problem in this paper, due to limit of the space, we are sorry for not being able to include the significant amount of references to other MAB variants. 

In a previous work~\cite{locatelli2016optimal}, the authors introduced the APT (Anytime Parameter-free
Thresholding) algorithm with the goal of simple regret minimization. In this algorithm, a precision parameter $\eps$ is used to determine the tolerance of errors (a.k.a.\ the indifference zone); and the APT algorithm only attempts to correctly classify the arms with hardness gap $\Delta_i > \eps$. This variant goal of simple regret partly alleviates the trivialization problem mentioned previously because of the extremely hard arms. In details, at any time $t$, APT selects the arm that minimizes $\sqrt{T_i(t)}\hat{\Delta}_i(t)$, where $T_i(t)$ is the number of times arm $i$ has been pulled until time $t$, $\hat{\Delta}_i(t)$ is defined as $|\hat{\theta}_i(t) -\theta| + \eps$, and $\hat{\theta}_i(t)$ is the empirical mean reward of arm $i$ at time $t$. In their experiments, ~\citet{locatelli2016optimal} also adapted the UCBE algorithm from~\cite{AudibertBM10} for the TBP problem and showed that APT performs better than UCBE. 

When the goal is to minimize the aggregate regret, the APT algorithm also works better than UCBE. However, we notice that the choice of precision parameter $\eps$ has significant influence on the algorithm's performance. A large $\eps$  makes sure that, when the sample budget is limited, the APT algorithm is not intrigued by a hard arm to spend overwhelmingly many samples on it without achieving a confident label. However, when the sample budget is ample, a large $\eps$ would also prevent the algorithm from making enough samples for the arms with hardness gap $\Delta_i < \eps$.  
Theoretically, the optimal selection of this precision parameter $\epsilon$ may differ significantly across the instances, and also depends on the budget $T$. In this work, we propose an algorithm that does not require such a precision parameter and demonstrates improved robustness in practice. 

Another natural approach to TBP is the uniform sampling method, where the learner plays each arm the same number of times (about $T/K$ times). In Appendix~\ref{app:hard-instance-uniform-sampling}, we show that the uniform sampling approach may need $\Omega(K)$ times more budget than the optimal algorithm to achieve the same aggregate regret.

Finally, \citet{chen2015statistical} proposed the optimistic knowledge gradient heuristic algorithm for budget allocation in crowdsourcing binary classification with Beta priors, which is closely related to the TBP problem in the Bayesian setting.

\vspace{-2ex}
\paragraph{Our Results and Contributions.}


Let $\mathcal{R}^{\mathbb{A}}(I; T)$ denote the aggregate regret of an instance $I$ after $T$ time steps. Given a sequence of hardness parameters $\Delta_1, \Delta_2, \dots, \Delta_K$, assume $\mathcal{I}_{\Delta_1, \dots, \Delta_K}$ is the class of all $K$-arm instances where the gap between $\theta_i$ of the $i$-th arm and the threshold $\theta$ is $\Delta_i$, and let 

\vspace{-3ex}
\begin{align}
    \mathrm{OPT}(\{\Delta_i\}_{i = 1}^{K}, T) \eqdef \inf_{\mathbb{A}} \sup_{I \in \mathcal{I}_{\Delta_1, \dots, \Delta_K}} \mathcal{R}^{\mathbb{A}}(I; T)
\end{align}

\vspace{-2ex}
be the minimum possible aggregate regret that any algorithm can achieve among all instances with the given set of gap parameters. We say an algorithm $\mathbb{A}$ is \emph{instance-wise asymptotically optimal} if for every $T$, any set of gap parameters $\{\Delta_i\}_{i = 1}^{K}$, and any instance $I \in \mathcal{I}_{\Delta_1, \dots, \Delta_K}$, it holds that 
\begin{align}
\mathcal{R}^{\mathbb{A}} (I; T) \leq O(1) \cdot \mathrm{OPT}(\{\Delta_i\}_{i = 1}^{K}, \Omega(T)).
\end{align}
While it may appear that a constant factor multiplied to $T$ can  affect the regret if the optimal regret is an exponential function of $T$, we note that our definition aligns with major multi-armed bandit literature (e.g., fixed-budget best arm identification \cite{GabillonGL12,carpentier2016tight} and thresholding bandit with simple regret \cite{locatelli2016optimal}). Indeed, according to our definition, if the universal optimal algorithm requires a budget of $T$ to achieve $\epsilon$ regret, an asymptotically optimal algorithm requires a budget of only $T$ multiplying some constant to achieve the same order of regret. On the other hand, if one wishes to pin down the optimal constant before $T$, even for the single arm case, it boils down to the question of the optimal (and distribution dependent) constant in the exponent of existing concentration bounds such as Chernoff Bound, Hoeffding's Inequality, and Bernstein Inequalities, which is beyond the scope of this paper.

We address the challenges mentioned previously and introduce a simple and elegant algorithm, the Logarithmic-Sample Algorithm (\LSA). \LSA has a very similar form as the APT algorithm in \cite{locatelli2016optimal} but introduces an additive term that is proportional to the logarithm of the number of samples made to each arm in order to more carefully allocate the budget among the arms (see Line~\ref{line:LSA-4} of Algorithm~\ref{alg:HTA}). This logarithmic term arises from the optimal sample allocation scheme of an offline algorithm when the gap parameters are known beforehand.  The log-sample additive term of \LSA can be interpreted as an incentive to encourage the samples for arms with bigger gaps and/or less explored arms, which boasts a similar idea as the incentive term in the famous Upper Confidence Bound (UCB) type of algorithms that date back to \citep{lai1985asymptotically, agrawal1995sample, auer2002using}, while interestingly the mathematical forms of the two incentive terms are very different.

As the main theoretical result of this paper, we analyze the aggregate regret upper bound of \LSA in Theorem~\ref{thm:upper-bound}. We complement the upper bound result with a lower bound theorem (Theorem~\ref{thm:lower-bound}) for any online algorithm. In Remark~\ref{rem:uppr-bound}, we compare the upper and lower bounds and show that \LSA is instance-wise asymptotically optimal.


We now highlight the technical contributions made in our regret upper bound analysis at a very high level. Please refer to Section~\ref{sec:upper-bound} for more detailed explanations. In our proof of the upper bound theorem, we first define a global class of events $\{\mathcal{F}_C\}$ (in \eqref{eq:F-C-def}) which serves as a measurement of how well the arms are explored. Our analysis then goes by two steps. In the first step, we show that $\mathcal{F}_C$ happens with high probability, which intuitively means that all arms are ``well explored''.  In the second step, we show the quantitative upper bound on the mis-classification probability for each arm, when conditioned on $\mathcal{F}_C$. The final regret bound follows by summing up the mis-classification probability for each arm via linearity of expectation. Using this approach, we successfully by-pass the analysis that involves pairs of (or even more) arms, which usually brings in union bound arguments and extra $\ln K$ terms. Indeed, such $\ln K$ slack appears between the upper and lower bounds proved in \cite{locatelli2016optimal}. In contrast, our \LSA algorithm is asymptotically optimal, without any super-constant slack.

Another important technical ingredient that is crucial to the asymptotic optimality analysis is a new concentration inequality for the empirical mean of an arm that uniformly holds over all time periods, which we refer to as the \emph{Variable Confidence Level Bound}. This new inequality helps to reduce an extra $\ln \ln T$ factor in the upper bound. It is also a strict improvement of the celebrated Hoeffding's Maximal Inequality, which might be useful in many other problems.

Finally, we highlight that our \LSA is \emph{anytime}, i.e., it does not need to know the time horizon $T$ beforehand. \LSA does use a universal tuning parameter. However, this parameter does not depend on the instances. As we will show in Section~\ref{sec:experiments}, the choice of the parameter is quite robust; and the natural parameter setting leads to superior performance of \LSA among a set of very different instances, while APT may suffer from poor performance if the precision parameter is not chosen well for an instance.

\vspace{-2ex}
\paragraph{Organization.} The organization of the rest of the paper is as follows. In Section~\ref{sec:formulation} we provide the necessary notation and definitions. Then we present the details of the $\LSA$ algorithm in Section~\ref{sec:algorithm} and upper bound its aggregate regret in Section~\ref{sec:upper-bound}. In Section~\ref{sec:experiments}, we present experiments establishing the empirical advantages of $\LSA$ over other algorithms. The instance-wise aggregate regret lower bound theorem is deferred to Appendix~\ref{sec:lower-bound}.

\vspace{-2ex}
\section{Problem Formulation and Notation} \label{sec:formulation}

\vspace{-2ex}
Given an integer $K>1$, we let $S = [K] \eqdef \{1, 2, \dots, K\}$ be the set of $K$ arms in an instance $I$. Each arm $i \in S$ is associated with a distribution $\mathcal{D}_i$ supported on $[0, 1]$ which has an unknown mean $\theta_i$. We are interested in the following dynamic game setting: At any round $t \geq 1$, the learner chooses to pull an arm $i_t$ from $S$ and receives an \textit{i.i.d.}\ reward sampled from $\mathcal{D}_{i_t}$.

We let $T$, with $T\geq K$, be the \emph{time horizon}, or the \emph{budget} of the game, which is not necessarily known beforehand. We furthermore let $\theta \in (0, 1)$ be the \emph{threshold} of the game. After $T$ rounds, the learner $\mathbb{A}$ has to determine, for every arm $i\in S$, whether or not its mean reward is greater than or equal to $\theta$. So the learner outputs a vector $(d_1,\ldots,d_K)\in\{0,1\}^K$, where $d_i=0$ if and only if $\mathbb{A}$ decides that $\theta_i<\theta$. The goal of the Thresholding Bandit Problem (TBP) in this paper is to  maximize the expected number of correct labels after $T$ rounds of the game. 

More specifically, for any algorithm $\mathbb{A}$, we use $\mathcal{E}_i^{\mathbb{A}}(T)$ to denote the event that $\mathbb{A}$'s decision corresponding to arm $i$ is correct after $T$ rounds of the game. The goal of the TBP algorithm is to minimize the \emph{aggregate regret}, which is the expected number of incorrect classifications for the $K$ arms, i.e., 
\begin{equation} \label{goal-form-0}
\mathcal{R}^{\mathbb{A}}(T) = \mathcal{R}^{\mathbb{A}}(I; T)\eqdef \E\left[ \sum_{i = 1}^K \mathbb{I}_{\{ \overline{\mathcal{E}}_i^{\mathbb{A}}(T) \}} \right],
\end{equation}
where $\overline{\mathcal{E}}$ denotes the complement of event $\mathcal{E}$ and $\mathbb{I}_{\{ \textrm{condition} \}}$ denotes the indicator function.

Let $X_{i, t}$ denote the random variable representing the sample received by pulling arm $i$ for the $t$-th time. We further write
\vspace{-3ex}
\begin{align}
    \hat{\theta}_{i,t} \eqdef \frac{1}{s} \sum_{s = 1}^t X_{i, s} ~~ \text{and}~~ \hat{\Delta}_{i, t} \eqdef | \hat{\theta}_{i,t} - \theta |
\end{align}
to denote the \emph{empirical mean} and the \emph{empirical gap} of arm $i$ after being pulled $t$ times, respectively. For a given algorithm $\mathbb{A}$, let $T_i^{\mathbb{A}}(t)$ and $\hat{\theta}_i^{\mathbb{A}}(t)$ denote the number of times arm $i$ is pulled and the empirical mean reward of arm $i$ after $t$ rounds of the game, respectively. For each arm $i \in S$, we use $\hat{\Delta}_i^{\mathbb{A}}(t) \eqdef |\hat{\theta}_i^{\mathbb{A}}(t) - \theta|$ to denote the empirical gap after $t$ rounds of the game. We will omit the reference to $\mathbb{A}$ when  it is clear from the context.

\vspace{-2ex}
\section{Our Algorithm} \label{sec:algorithm}

\vspace{-2ex}
We now motivate our Logarithmic-Sample Algorithm by first designing an optimal but unrealistic algorithm with the assumption that the hardness gaps $\{\Delta_i\}_{i\in S}$ are known beforehand. Now we design the following algorithm $\mathbb{O}$. Suppose the algorithm pulls arm $i$ a total of $x_i$ times and makes a decision based on the empirical mean $\hat{\theta}_{i, x_i}$: if $\hat{\theta}_{i, x_i} \geq \theta$, the algorithm decides that $\theta_i \geq \theta$, and decides $\theta_i < \theta$ otherwise. Note that this is all algorithm can do when the gaps $\Delta_i$ are known. We upper bound the aggregate regret of the algorithm by 
\begin{align} \label{eq-1:subsec:intuition}
\mathcal{R}^{\mathbb{O}}(T) & = \sum_{i = 1}^K \Pr (  \overline{\mathcal{E}}_i^{ \mathbb{O} }(T) ) 
\leq \sum_{i = 1}^K \Pr (  | \hat{\theta}_{i, x_i} - \theta_i | \geq \Delta_i )
 \leq \sum_{i = 1}^K 2 \exp\left( - 2 x_i \Delta_i^2 \right),
\end{align}
where the last inequality follows from Chernoff-Hoeffding Inequality (Proposition~\ref{prop:chernoff-hoeff}).
Now we would like to minimize the RHS (right-hand-side) of \eqref{eq-1:subsec:intuition}, and upper bound the aggregate regret of the optimal algorithm $\mathbb{O}$ by 
\vspace{-2ex}
\[
2 \cdot  \min_{ \substack{x_1 + \cdots + x_K = T \\ x_1,\ldots,x_K \in \mathbb{N} } } \sum_{i = 1}^K \exp( - 2 x_i \Delta_i^2 ) = 2\mathscr{P}^*_{2} (\{\Delta_i\}_{i\in S}, T).
\]
Here, for every $c > 0$, we define
\vspace{-2ex}
\begin{align}
    \mathscr{P}^*_{c} (\{\Delta_i\}_{i\in S}, T) \eqdef  \min_{ \substack{x_1 + \cdots + x_K = T \\ x_1,\ldots,x_K \in \mathbb{N} } } \sum_{i = 1}^K \exp( - c x_i \Delta_i^2 ).
\end{align}
\vspace{-2ex}

We naturally introduce the following continuous relaxation of the program $\mathscr{P}_{c}$, by defining
\begin{align}
    \mathscr{P}_{c} (\{\Delta_i\}_{i\in S}, T) \eqdef  \min_{ \substack{x_1 + \cdots + x_K = T \\ x_1,\ldots,x_K \geq 0 } } \sum_{i = 1}^K \exp( - c x_i \Delta_i^2 ). 
\end{align}
$\mathscr{P}_{c}$ well approximates $\mathscr{P}^*_{c}$, as it is straightforward to see that 
\begin{equation} \label{eq:Pc-approximates-Pcstar}
    \mathscr{P}_{c} (\{\Delta_i\}_{i\in S}, T) \leq \mathscr{P}_{c}^* (\{\Delta_i\}_{i\in S}, T) 
    \leq \mathscr{P}_{c} (\{\Delta_i\}_{i\in S}, T - K) .
\end{equation} 

We apply the Karush-Kuhn-Tucker (KKT) conditions to the optimization problem $\mathscr{P}_{2} (\{\Delta_i\}_{i\in S}, T)$ and find that the optimal solution satisfies
\begin{align}\label{eq:opt-sol-rel-Pc}
x_i \Delta_i^2 + \ln \Delta_i^{-1} \geq \Phi, \text{for~} i \in S,
\end{align}
where $ \Phi \eqdef \argmax_x \{ x : \sum_{i = 1}^K \max\{  \frac{ x - \ln \Delta_i^{-1} }{\Delta_i^2} , 0 \} \leq T \} $ is independent of $i\in S$.  Furthermore, since $\sum_{i = 1}^K \max\{  \frac{ x - \ln \Delta_i^{-1} }{\Delta_i^2} , 0 \} $ is an increasing continuous function on $x$, $\Phi$ is indeed well-defined. Please refer to Lemma~\ref{lem:solution-of-programming} of  Appendix~\ref{app-sec:missing-computation} for the details of the relevant calculations.

In light of \eqref{eq:Pc-approximates-Pcstar} and \eqref{eq:opt-sol-rel-Pc}, the following algorithm $\mathbb{O}'$ (still, with the unrealistic assumption of the knowledge of the gaps $\{\Delta_i\}_{i\in S}$) incrementally solves $\mathscr{P}_{c}$ and approximates the algorithm $\mathbb{O}$ -- at each time $t$, the algorithm selects the arm $i$ that minimizes $T_i(t-1) \Delta_i^2 + \ln \Delta_i^{-1}$ and plays it.

Our proposed algorithm is very close to $\mathbb{O}'$. Since in reality the algorithm does not have access to the precise gap quantities, we use the empirical estimates $\hat{\Delta}_i^2$ instead of $\Delta_i^2$ in the $T_i(t-1) \Delta_i^2$ term. For the logarithmic  term, if we also use $\ln \hat{\Delta}_i^{-1}$ instead of $\ln \Delta_i^{-1}$, we may encounter extremely small empirical estimates when the arm is not sufficiently sampled, which would lead to unbounded value of $\ln \hat{\Delta}_i^{-1}$, and render the arm almost impossible to be sampled in future. To solve this problem, we note that  $\mathbb{O}'$ tries to maintain $T_i(t-1) \Delta_i^2$ to be roughly the same across the arms (if ignoring the $\ln \Delta_i^{-1}$ term). In light of this, we use $\sqrt{T_i(t-1)}$ to roughly estimate the order of $\Delta_i^{-1}$. This encourages the exploration of both the arms  with larger gaps and the ones with fewer trials.

To summarize, at each time $t$, our algorithm selects the arm $i$ that minimizes $\alpha \cdot T_i(t - 1) ( \hat{\Delta}_i( t - 1) )^2 + 0.5 \ln T_i(t - 1)$, where $\alpha > 0$ is a universal tuning parameter, and plays the arm.  The details of the algorithm are presented in Algorithm~\ref{alg:HTA}.  
\vspace{-2ex}
\begin{algorithm}[H]
\caption{Logarithmic-Sample Algorithm, $\LSA(S, \theta)$}
\label{alg:HTA}
\begin{algorithmic}[1]  
\State \textbf{Input:} A set of arms $S = [K]$, threshold $\theta$ 
\State \textbf{Initialization:} pull each arm once
\For {$t = K + 1 \textbf{~to~} T$}
\State Pull arm
$\displaystyle{
i_t =
 \argmin\limits_{i \in S} \left( \alpha T_i(t - 1)  ( \hat{\Delta}_i(t - 1) )^2 + 0.5 \ln T_i(t - 1) \right)
 }
$ \label{line:LSA-4}
\EndFor
\State For each arm $i \in S$, let $d_i \leftarrow 1$ if $\hat{\theta}_i(T) \geq \theta$ and $d_i \leftarrow 0$ otherwise
\State \textbf{Output:} $(d_1, \dots, d_K )$
\end{algorithmic}
\end{algorithm}

\vspace{-3ex}
\section{Regret Upper Bound for \LSA}\label{sec:upper-bound}
\vspace{-2ex}

In this section, we show the upper bound of the aggregate regret of Algorithm~\ref{alg:HTA}.

Let $x=\Lambda$ be the solution to the following equation 
\begin{equation} \label{eq:lambda-def-1}
    \sum_{i = 1}^K \Bigg( \mathbb{I}_{\{ x \leq \ln \Delta_i^{-1} \}} \cdot \exp(2x) \\ + \mathbb{I}_{\{ x > \ln \Delta_i^{-1} \}} \cdot \frac{x - \ln \Delta_i^{-1} + \alpha }{\alpha \Delta_i^2} \Bigg) = \frac{T}{ \max\{ 40/\alpha + 1, 40\}}.
\end{equation} 
Notice that $\sum_{i = 1}^K ( \mathbb{I}_{\{  x \leq \ln \Delta_i^{-1} \}} \cdot  \exp(2x) + \mathbb{I}_{\{ x > \ln \Delta_i^{-1} \}} \cdot \frac{x - \ln \Delta_i^{-1} + \alpha }{\alpha \Delta_i^2} )$ is a strictly increasing, continuous function with $x \geq 0$ that becomes $K$ when $x = 0$ and goes to infinity when $x \rightarrow \infty$. Hence $\Lambda$ is guaranteed to exist and is uniquely defined when $T$ is large.  Furthermore, for any $i\in S$, we let 
\begin{equation} \label{eq:lambda-def-2}
\lambda_i \eqdef \mathbb{I}_{\{  \Lambda \leq \ln \Delta_i^{-1} \}} \cdot  \exp(2 \Lambda) \\ 
+ \mathbb{I}_{\{ \Lambda > \ln \Delta_i^{-1} \}} \cdot \frac{ \Lambda - \ln \Delta_i^{-1} + \alpha }{ \alpha \Delta_i^2 }.
\end{equation}
We note that $\{ \lambda_i \}_{i \in S}$ is the optimal solution to $\mathscr{P}_{2\alpha} (\{\max\{ \Delta_i, \exp(- \Lambda) \}\}_{i\in S}, T/(\max\{ 40/\alpha + 1, 40\}))$. Please refer to Lemma~\ref{lem:solution-of-a-special-programming} of  Appendix~\ref{app-sec:missing-computation} for the detailed calculations.

The goal of this section is to prove the following theorem.

\begin{theorem} \label{thm:upper-bound}
Let $\mathcal{R}^{\LSA}(T)$ be the aggregate regret incurred by Algorithm~\ref{alg:HTA}. When $0 < \alpha \leq 8$, and $T \geq \max\{40/\alpha + 1, 40\} \cdot K $, we have 
\begin{equation} \label{eq:thm-main}
\mathcal{R}^{\LSA}(T) \leq \Phi(\alpha) \cdot \sum_{i \in S} \exp\left( - \frac{ \lambda_i  \Delta_i^2}{10} \right),
\end{equation}
where $\Phi(\alpha) = \frac{9.3 \cdot \sqrt[8\alpha]{2} }{ \sqrt[8\alpha]{2} - 1}  \exp \left( \frac{2.1 \alpha - \ln \alpha - 0.5}{ 4\alpha} \right) $ is a constant that only depends on the universal tuning parameter $\alpha$.
\end{theorem}

\begin{remark} \label{rem:uppr-bound}
If we set $\alpha = 1/20$, then the right-hand side of \eqref{eq:thm-main} would be at most $O\left( \sum_{i \in S} \exp\left(  - \frac{\lambda_i \Delta_i^2}{10}  \right)\right)$. One can verify that
\begin{align*}
 &\sum_{i \in S} {} \exp\left(  - \frac{\lambda_i \Delta_i^2}{10} \right) 
 \leq  {}  O \left(\mathscr{P}_{1/10} (\{\max\{ \Delta_i, \exp(- \Lambda) \}\}_{i\in S}, T/801)\right) \\
& \qquad =  {}  O \left(\mathscr{P}_{16} (\{\max\{ \Delta_i, \exp(- \Lambda) \}\}_{i\in S}, T/ 128160 )\right) 
\leq {}  O \left({\mathscr{P}}_{16} (\{ \Delta_i\}_{i\in S}, T/128160) \right) .
\end{align*}
where the first inequality is due to Lemma~\ref{lem:relationship-of-two-programmings} of Appendix~\ref{app-sec:missing-computation} and the equality is because of Lemma~\ref{lem:change-c-of-a-programming} of Appendix~\ref{app-sec:missing-computation}.
This matches the lower bound demonstrated in Theorem~\ref{thm:lower-bound} up to constant factors. \footnote{While the constants may seem large, we emphasize that {\romannumeral 1}) we make no effort in optimizing the constants in asymptotic bounds, {\romannumeral 2}) most of the constants come from the lower bound, while the constant factor in our upper bound is $10$, and {\romannumeral 3}) we believe that the actual constant of our algorithm is quite small, as the experimental evaluation in the later section demonstrates that our algorithm performs very well in practice.}
\end{remark}

The rest of this section is devoted to the proof of Theorem~\ref{thm:upper-bound}. Before proceeding, we note that the analysis of the APT algorithm \cite{locatelli2016optimal} crucially depends on a favorable event stating that the empirical mean of any arm at any time does not deviate too much from the true mean. This requires a union bound that introduces extra factors such as $\ln K$ and $\ln \ln T$. Our analysis adopts a novel approach that does not need a  union bound over all arms, and hence avoids the extra $\ln K$ factor. In the second step of our analysis, we introduce the new \emph{Variable Confidence Level Bound} to save the extra doubly logarithmic term in $T$.

Now we dive into details of the proof. Let $B \eqdef \{ i \in S: \ln \Delta_i^{-1} < \Lambda \}$. Intuitively, $B$ contains the arms that can be well classified by the ideal algorithm $\mathbb{O}$ (described in Section~\ref{sec:algorithm}), while even the ideal algorithm $\mathbb{O}$ suffers $\Omega(1)$ regret for each arm in $S \setminus B$. In light of this, the key of the proof is to upper bound the regret incurred by the arms in $B$.

Let $\mathcal{R}_B^{\LSA}(T)$ denote the regret incurred by arms in $B$. Note that $\Phi(\alpha) \cdot \exp(-\lambda_i \Delta_i^2/10) \geq 1$ for every arm $i \in S \setminus B$, and the regret incurred by each arm is at most $1$. Therefore, to establish (\ref{eq:thm-main}), we only need to show that 
\vspace{-2ex}
\begin{equation} \label{eq:main-thm-1}
    \mathcal{R}_B^{\LSA}(T)
    \leq \Phi(\alpha) \cdot \sum_{i \in B} \exp\left(\frac{ -  \lambda_i \Delta_i^2}{10} \right)  .
\end{equation}
\vspace{-2ex}

We set up a few notations to facilitate the proof of \eqref{eq:main-thm-1}. We define 
$\stat_i(t) \eqdef \alpha T_i(t)  ( \hat{\Delta}_i(t) )^2 + 0.5 \ln T_i(t)$
to be the expression inside the $\argmin(\cdot)$ operator in Line \ref{line:LSA-4} of the algorithm, for arm $i$ and at time $t$. We also define
$\stat_{i, t} \eqdef \alpha t  ( \hat{\Delta}_{i, t} )^2 + 0.5 \ln t$.

Intuitively, when $\stat_i(t)$ is large, we usually have a larger value for $T_i(t)$, and arm $i$ is better explored. Therefore, $\stat_i(t)$ can be used as a measurement of how well arm $i$ is explored, which directly relates to the mis-classification probability for classifying the arm. We say that arm $i$ is \emph{$C$-well explored} at time $T$ if there exists $T' \leq T$ such that $\stat_i(T') > C$. For any $C > 0$, we also define the event $\mathcal{F}_C$ to be
\vspace{-2ex}
\begin{align}\label{eq:F-C-def}
    \mathcal{F}_C \eqdef \left\{\exists T' \leq T\ :\ \forall i \in S, \stat_i(T') > C \right\} .
\end{align} 
When $\mathcal{F}_C$ happens, we know that all arms are $C$-well explored.

At a higher level, the proof of \eqref{eq:main-thm-1} goes by two steps. First, we show that for $C$ that is almost as large as $\Lambda$, $\mathcal{F}_C$ happens with high probability, which means that every arm is $C$-well explored. Second, we quantitatively relate that being $C$-well explored and the mis-classification probability for classifying each arm, which can be used to further deduce a regret upper bound given the event $\mathcal{F}_C$.

We start by revealing more details about the first step. The following Lemma~\ref{lem:good-event-prob} gives a lower bound on the probability of the event $\mathcal{F}_C$.

\begin{lemma} \label{lem:good-event-prob}
$\Pr( \mathcal{F}_{\Lambda - k} ) \geq 1 - \exp(-40k/\alpha)$ for $0 \leq k < \Lambda $.
\end{lemma}

We now introduce the high-level ideas for proving Lemma~\ref{lem:good-event-prob} and defer the formal proofs to Appendix~\ref{sec:ub-step-1}. For any arm $i \in S$ and $C > 0$, let $\tau_{i, C}$ be the random variable representing the smallest positive integer such that $\stat_{i, \tau_{i, C}} > C$ (i.e., $\stat_{i, t} \leq C$ for all $1\leq t<\tau_{i,C}$). Intuitively, $\tau_{i, C}$ denotes the first time arm $i$ is $C$-well explored. We first show that the distribution of $\tau_{i, C}$ has an exponential tail. Hence, the sum of them with the same $C$ also has an exponential tail. Next, we show that with high probability $\sum_{i = 1}^K \tau_{i, \Lambda - k} \leq T$ and the probability vanishes exponentially as $k$ increases. In the last step, thanks to the design of the algorithm, we are able to argue that $\sum_{i = 1}^K \tau_{i, \Lambda - k} \leq T$ implies $\mathcal{F}_{\Lambda - k}$.

We now proceed to the second step of the proof of \eqref{eq:main-thm-1}. The following lemma (whose proof is deferred to Appendix~\ref{sec:ub-step-2}) gives an upper bound of regret incurred by arms in $B$ conditioned on $\mathcal{F}_C$. 

\begin{lemma} \label{lem:regret-incurred-by-bad-arm}
If $k \geq 0.1 \alpha$, then conditioned on $\mathcal{F}_{\Lambda - k}$,
\[
\mathcal{R}_B^{\LSA}(T) \leq \frac{9 \cdot   \sqrt[8\alpha]{2} }{ \sqrt[8\alpha]{2} - 1}  \cdot \sum_{i \in B}  \exp \left( - \frac{\lambda_i \Delta_i^2}{10}  + \frac{k + \alpha - \ln \alpha - 0.5}{4 \alpha} \right) .
\]
\end{lemma}

As mentioned before, the key to proving Lemma~\ref{lem:regret-incurred-by-bad-arm} is to pin down the quantitative relation between the event $\mathcal{F}_C$ and the probability of mis-classifying an arm conditioned on $\mathcal{F}_C$, then the expected regret upper bound can be achieved by summing up the mis-classifying probability for all arms in $B$. 

A key technical challenge in our analysis is to design a concentration bound for the empirical mean of an arm (namely arm $i$) that uniformly holds over all time periods. A typical method is to let the length of the confidence band scale linearly with $\sqrt{1 / t}$, where $t$ is the number of samples made for the arm. However, this would worsen the failure probability, and lead to an extra $\ln \ln T$  factor in the regret upper bound. To reduce the iterated logarithmic factor, we introduce a novel uniform concentration bound where the ratio between the length of the confidence band and $\sqrt{1 / t}$ is almost constant for large $t$, but becomes larger for  smaller $t$. Since this ratio is related to the confidence level of the corresponding confidence band, we refer to this new concentration inequality as the \emph{Variable Confidence Level Bound}. More specifically, in Appendix~\ref{sec:proof-ub-step2-1}, we prove the following lemma.

\begin{namedthm*}{Lemma~\ref{lem:variational-lil} \textnormal{({\it Variable Confidence Level Bound}, { pre-stated})}} 
Let $X_1, \dots, X_L$ be \textit{i.i.d.}\ random variables supported on $[0, 1]$ with mean $\mu$. For any $a > 0$ and $b > 0$, it holds that
\vspace{-1ex}
\[
\Pr\left( \forall t \in [1, L], \left| \frac{1}{t} \sum_{i = 1}^{t} X_i - \mu \right| \leq  \sqrt{ \frac{ a + b \ln (L/t) }{t} } \right) \geq 1 - \frac{2^{b/2+2} }{ 2^{b/2} - 1}  \exp (- a/2).
\] 
\vspace{-1ex}
\end{namedthm*}

This new inequality greatly helps the analysis of our algorithm, where the intuition is that when conditioned on the event $\mathcal{F}_C$, it is much less likely that fewer number of samples are conducted for arm $i$, and therefore we can afford a less accurate (i.e.\ bigger) confidence band for its mean value. 

It is notable that a similar idea is also adopted in the analysis of the MOSS algorithm  \cite{DBLP:conf/colt/AudibertB09} which gives the asymptotically optimal regret bound for the ordinary multi-armed bandits. However, our Variable Confidence Level Bound is more general and may be useful in other applications. We additionally remark that in the celebrated Hoeffding's Maximal Inequality, the confidence level also changes with time. However, the blow-up factor made to the confidence level in our inequality is only the logarithm of that of the Hoefdding's Maximal Inequality. Therefore, if constant factors are ignored, our inequality strictly improves Hoeffding's Maximal Inequality.

The formal proof of Theorem~\ref{thm:upper-bound} involves a few technical tricks to combine Lemma~\ref{lem:good-event-prob} and Lemma~\ref{lem:regret-incurred-by-bad-arm} to deduce the final regret bound, and is deferred to Appendix~\ref{app:proof-main-thm-ub}. The lower bound theorem (Theorem~\ref{thm:lower-bound}) that complements Theorem~\ref{thm:upper-bound} is deferred to Appendix~\ref{sec:lower-bound} due to space constraints.

\vspace{-2ex}
\section{Experiments}\label{sec:experiments}
\vspace{-2ex}

In our experiments, we assume that each arm follows independent Bernoulli distributions with different means. To guarantee a fair comparison, we vary the total number of samples $T$ and compare the empirical average aggregate regret on a logarithmic scale which is averaged over $5000$ independent runs. We consider three different choices of $\{ \theta_i \}_{i \in S}$:

\begin{enumerate}
  \setlength{\itemsep}{0pt}
  \setlength{\parskip}{0pt}
\item[1.] (arithmetic progression I). $K = 10$; $\theta_{1:4} = 0.2 + (0:3) \cdot 0.05, \theta_5 = 0.45, \theta_6 = 0.55$, and $\theta_{7:10} = 0.65 + (0:3) \cdot 0.05$ (see Setup 1 in Figure~\ref{fig:exp1}).
\item[2.](arithmetic progression II). $K = 20$; $\theta_{1:20} = 0.405 + (i - 1) / 100$ (see Setup 2 in Figure~\ref{fig:exp1}).
\item[3.] (two-group setting). $K = 10$; $\theta_{1:5} = 0.45$, and $\theta_{6:10} = 0.505$ (see Setup 3 in Figure~\ref{fig:exp1}).
\end{enumerate}

\begin{figure*}[t]
\centering
  \includegraphics[width=\textwidth]{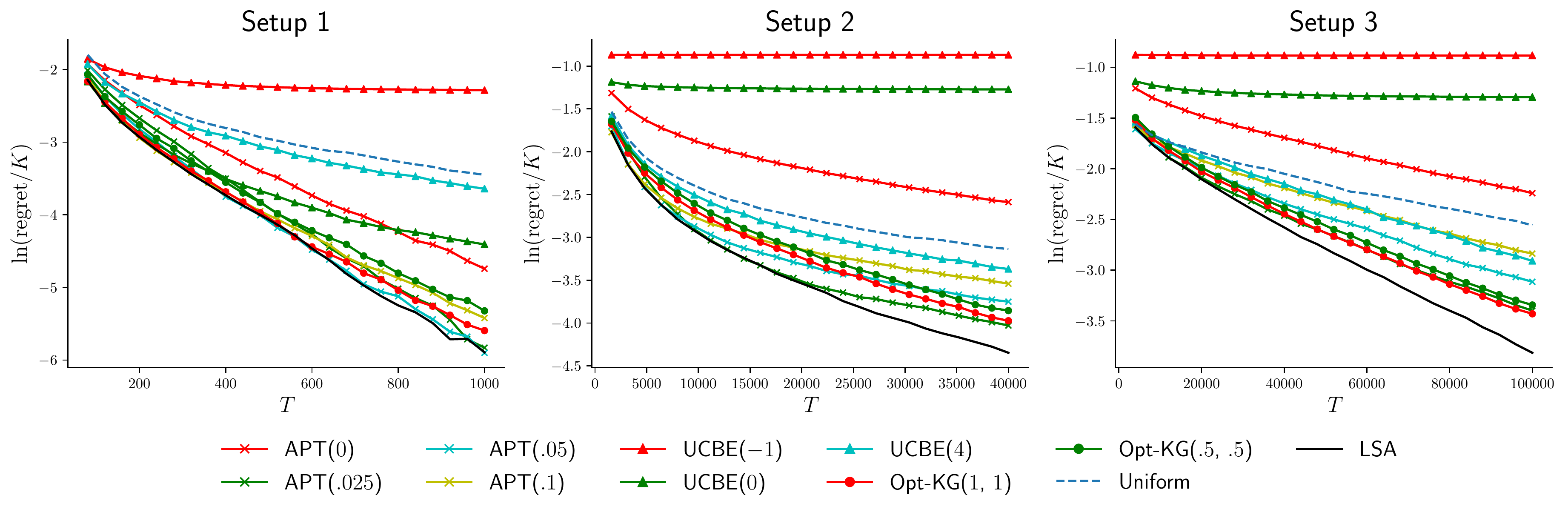}
\caption{Average aggregate regret on a logarithmic scale for different settings.}
\label{fig:exp1}
\end{figure*}

\vspace{-2ex}

In our experiments, we fix $\theta = 0.5$. We notice that the choice of $\alpha$ in our \LSA  is quite robust (see Appendix~\ref{app:more-experiments-b} for experimental results). To illustrate the performance, we fix $\alpha = 1.35$ in \LSA and compare it with four existing algorithms for the TBP problem under a variety of settings. Now we discuss these algorithms and their parameter settings in more details.
\begin{itemize}
  \setlength{\itemsep}{0pt}
  \setlength{\parskip}{0pt}
  \setlength{\parindent}{0pt}
\item{\bf Uniform:} Given the budget $T$, this method pulls each arm sequentially from $1$ to $K$ until budget $T$ is reached such that each arm is sampled roughly $T/K$ times. Then it outputs $\theta_i \geq \theta$ when $\hat{\theta}_i \geq \theta$.  
\item\textbf{APT($\eps$):} Introduced and analyzed in \cite{locatelli2016optimal}, this algorithm aims to output a set of arms ($\{i \in S : \hat{\mu}_i \geq \theta \}$) serving as an estimate of the set of arms with means over $\theta + \eps$. The natural adaptation of the APT algorithm to  our problem corresponds to changing the output: it outputs $\theta_i \geq \theta$ if $\hat{\theta}_i \geq \theta$ and $\theta_i < \theta$ otherwise. In the experiments, we test the following  choices of $\eps$: $0$, $0.025$, $0.05$, and $0.1$.
\item\textbf{UCBE($i$):} Introduced and analyzed in \cite{audibert2010best}, this algorithm aims to identify the best arm (the arm with the largest mean reward). A natural adaptation of this algorithm to TBP is for each time $t$, it pulls $\argmin_{i \in S} ( \hat{\Delta}_i - \sqrt{{a}/{T_i(t - 1)}}) $ where $a$ is a tuning parameter. In \cite{audibert2010best}, it has been proved optimal when $a = \frac{25}{36} \frac{T - K}{H}$ where $H = \sum_{i \in S} \frac{1}{\Delta_i^2}$. Here we set $a = 4^i \frac{T - K}{H}$ and test three different choices of $i$: $-1$, $0$, and $4$.
\item\textbf{Opt-KG($a$, $b$):} Introduced in \cite{chen2015statistical}, this algorithm also aims to minimize the aggregate regret. It models TBP as a Bayesian Markov decision process where $\{ \theta_i \}_{i \in S}$ is assumed to be drawn from a known Beta prior $\mathrm{Beta}(a, b)$. Here we choose two different priors: $\mathrm{Beta}(1, 1)$ (uniform prior) and $\mathrm{Beta}(0.5, 0.5)$ (Jeffreys prior).
\end{itemize}

\vspace{-2ex}

\paragraph{Comparisons.}
In Setup 1, which is a relatively easy setting, \LSA works best among all choices of budget $T$. With the right choice of parameter, APT and Opt-KG also achieve satisfactory performance. Though the performance gaps appear to be small, two-tailed paired t-tests of aggregate regrets indicate that \LSA is significantly better than most of the other methods, except APT(.05) and APT(.025) (see Table~\ref{tab:p-values} in Appendix~\ref{app:more-experiments-c}).

In Setup 2 and 3, where ambiguous arms close to the threshold $\theta$ are presented, the performance difference between \LSA and other methods is more noticeable. \LSA consistently outperforms other methods in both settings over almost all choices of budget $T$ with statistical significance. It is worth noting that, though APT works also reasonably well in Setup 2 when $T$ is small, the best parameter $\eps$ is different from that for bigger $T$ and other setups. On the other hand, the parameters chosen in \LSA are fixed across all setups, indicating that our algorithm is more robust.

We perform additional experiments that due to space limitations are included in Appendix~\ref{app:more-experiments-a}. In all setups, \LSA outperforms its competitors with various parameter choices.

\vspace{-2ex}
\section{Conclusion}
\vspace{-2ex}

In this paper we introduce an algorithm that minimizes the aggregate regret for the thresholding bandit problem. Our algorithm \LSA makes use of a novel approach inspired by the optimal allocation scheme of the budget when the reward gaps are known ahead of time. When compared to APT, \LSA uses an additional term, similar in spirit to the UCB-type algorithms though mathematically different, that encourages the exploration of arms that have bigger gaps, and/or those have not been sufficiently explored. Moreover, \LSA is anytime and robust, while the precision parameter $\epsilon$ needed in the APT algorithm is highly sensitive and hard to choose. Besides showing empirically that \LSA performs better than APT for different values of $\epsilon$ and other algorithms in a variety of settings, we also employ novel proof ideas that eliminate the logarithmic terms usually brought in by the straightforward union bound argument, design the new \emph{Variable Confidence Level Bound} that strictly improves the celebrated Hoeffding's Maximal inequality, and prove that $\LSA$ achieves instance-wise asymptotically optimal aggregate regret.

\bibliography{chaotao}

\newpage

\appendix

{\bf \LARGE \centering Appendix \par}

\section{Probability Tools}

\begin{prop}[Chernoff-Hoeffding Inequality \cite{hoeffding1963probability}]  \label{prop:chernoff-hoeff}
Let $\{ X_i \}_{i \in [t]}$ be a list of independent random variables supported on $[0, 1]$ and set $X = \frac{1}{t} \sum_{i = 1}^t X_i$. Then, for every $\eps > 0$, it holds that
\[
\Pr( |X - \E[X] | \geq \eps ) \leq 2 \exp( -2 t \eps^2 ).
\]
\end{prop}

\begin{prop}[Restatement of Theorem~5.1(ii) in \cite{janson2018tail}] \label{prop:tail-of-sum-of-expo}
Let $\{X_i\}_{i \in [t]}$ be a list of independent random variables such that $\Pr( X_i > x ) \leq \exp(- a_i x)$ for $x > 0$. And let $\mu = \sum_{i = 1}^t \frac{1}{a_i}$. Then for any $\lambda \geq 1$, it holds that
\[
\Pr( X \geq \lambda \mu ) \leq \exp(1 - \lambda).
\]
\end{prop}

\begin{prop}[Hoeffding's Maximal Inequality \cite{hoeffding1963probability}] \label{prop:hoeff-max-inequ} 
Let $\{ X_i \}_{i \in [t]}$ be a list of \textit{i.i.d.}\ random variables supported on $[0, 1]$ and set $\mu = \E[ X_1 ]$. Then, for any $\eps > 0$, it holds that
\[
\Pr( \forall i \in [t], X_1 + X_2 + \cdots + X_i \geq i \mu+ \eps ) \leq \exp\left( -\frac{2\eps^2}{t} \right).
\]
\end{prop}

\begin{prop}[Restatement of Lemma 2.6 in \cite{Tsybakov:2008:INE:1522486}] \label{prop:lower-bound-technique}
  Let $P$ and $Q$ be two probability distributions supported on some set $\mathcal{X}$. Then for every set $A \subset \mathcal{X}$, one has
  \begin{equation*}
    \Pr\nolimits_{X \sim P }(A) + \Pr\nolimits_{X \sim Q }(\overline{A}) \geq \frac{1}{2} \exp(-\KL(P \parallel Q)),
  \end{equation*}
where $\overline{A}$ denotes the complement of $A$ and $\KL$ denotes the Kullback-Leibler divergence between $P$ and $Q$ given by 
  \[ 
  \KL(P \parallel Q) \eqdef \sum_{x \in \mathcal{X}} P(x) \ln \left( \frac{P(x)}{Q(x)} \right).
  \]
\end{prop}

\begin{prop}[Restatement of Lemma 15.1 in \cite{LS18}] \label{prop:divergence-decomposition}
  Let $v = P_1 \otimes \cdots \otimes P_K$ and $v' = P'_1 \otimes \cdots \otimes P'_K$ be the reward distributions of two $K$-armed bandits. Assuming $\KL( P_i, P_i' ) < +\infty$ for any arm $i \in [K]$. Fix some policy $\pi$ and let $\Pr_v = \Pr_{v\pi}$ and $\Pr_{v} = \Pr_{v'\pi}$ be the two probability measures induced by the $n$-round interconnection of $\pi$ and $v$ (respectively, $\pi$ and $v'$). Then
  \begin{equation*}
    \KL( \Pr\nolimits_v \parallel \Pr\nolimits_{v'} ) = \sum_{i = 1}^K \E\nolimits_v[ T_i(n) ] \cdot \KL( P_i\parallel P_i' ),
  \end{equation*}
  where $T_i(n)$ is the random variable denoting the number of times arm $i$ is pulled.
\end{prop}

\section{Properties of  $\mathscr{P}_{c} $}\label{app-sec:missing-computation}

We first show the optimal solution to $\mathscr{P}_{c} (\{\Delta_i\}_{i\in S}, T)$ by proving the following lemma.

\begin{lemma} \label{lem:solution-of-programming}
If $c > 0$, then the optimal solution to $\mathscr{P}_{c} (\{\Delta_i\}_{i\in S}, T)$ can be expressed in the following form
\[
x_i = \max \left\{  \frac{ \Phi_c - \ln \Delta_i^{-1} }{c\Delta_i^2 / 2} , 0 \right\}, 
\]
where $ \Phi_c \eqdef \argmax_x \{ x : \sum_{i = 1}^K \max\{  \frac{ x - \ln \Delta_i^{-1} }{c\Delta_i^2 / 2} , 0 \} \leq T \} $.
\end{lemma}

\begin{proof}
Since $\sum_{i = 1}^K \max\{  \frac{ x - \ln \Delta_i^{-1} }{c\Delta_i^2 / 2} , 0 \} $ is an increasing continuous function on $x$, $\Phi_c$ is indeed well-defined.

We apply KKT conditions (see Proposition 8.7.2 in~\cite{MatouekG07}) to solve the minimization problem $\mathscr{P}_{c} (\{\Delta_i\}_{i\in S}, T)$. Concretely, the KKT conditions applies to $\mathscr{P}_{c} (\{\Delta_i\}_{i\in S}, T)$ gives 
\begin{align*}
  (-c\Delta_i^2) \exp(-c x_i \Delta_i^2) - u_i + v & = 0 \mathrm{~for~} i \in [K] \\
  u_i x_i & = 0 \mathrm{~for~} i \in [K] \\
  u_i & \leq 0 \mathrm{~for~} i \in [K]\\
  x_i & \geq 0 \mathrm{~for~} i \in [K] \\
   \sum_{i = 1}^K x_i&  = T ,
\end{align*}
where $u_i$ for $i \in [K]$ and $v$ are $K+1$ newly-introduced variables. In particular, if $x_i > 0$, then $u_i = 0$ and it holds that 
\begin{equation} \label{eq:lem:solution-of-programming}
\frac{c}{2} x_i\Delta_i^2 + \ln \Delta_i^{-1} = \frac{1}{2} \ln \frac{c}{v}.
\end{equation}
It is easy to see the solution $ x_i = \max\{  \frac{ \Phi_c - \ln \Delta_i^{-1} }{ c\Delta_i^2 / 2} , 0 \} $ for $i \in [K]$ satisfies \eqref{eq:lem:solution-of-programming} and is a minimum point.
\end{proof}

For any positive number $c > 0$, let $x = \Psi_c$ be the solution to
\begin{align*}
\sum_{i = 1}^K \Bigg( \mathbb{I}_{\{ x \leq \ln \Delta_i^{-1} \}} \cdot  \exp(2 x) + \mathbb{I}_{\{ x > \ln \Delta_i^{-1} \} } 
\cdot\frac{ x - \ln \Delta_i^{-1} + c/2}{ c \Delta_i^2 / 2} \Bigg)= T .
\end{align*}
Note that 
\begin{align*}
\sum_{i = 1}^K \Bigg( \mathbb{I}_{\{  x \leq \ln \Delta_i^{-1} \}} \cdot \exp(2x) + \mathbb{I}_{\{ x > \ln \Delta_i^{-1} \}} 
 \cdot \frac{x - \ln \Delta_i^{-1} + c/2}{ c \Delta_i^2 / 2} \Bigg)
\end{align*}
is a strictly increasing continuous function on $x$ that equals $K$ when $x = 0$ and tends to infinity when $x \rightarrow \infty$. Hence $\Psi_c$ exists and is uniquely defined.

Then we derive the optimal solution to $\mathscr{P}_{c} (\{\max\{ \Delta_i, \exp(- \Psi_c) \}\}_{i\in S}, T)$, as follows.

\begin{lemma} \label{lem:solution-of-a-special-programming}
If $c > 0$, then the optimal solution to $\mathscr{P}_{c} (\{\max\{ \Delta_i, \exp(- \Psi_c) \}\}_{i\in S}, T)$ can be expressed in the following form
\begin{align*}
x_i = \mathbb{I}_{\{  \Psi_c \leq \ln \Delta_i^{-1} \}} \cdot  \exp(2 \Psi_c ) + \mathbb{I}_{\{ \Psi_c > \ln \Delta_i^{-1} \}} 
\cdot \frac{\Psi_c - \ln \Delta_i^{-1} + c/2}{ c \Delta_i^2 / 2}.
\end{align*}
\end{lemma}

\begin{proof}
By Lemma~\ref{lem:solution-of-programming}, the optimal solution to $\mathscr{P}_{c} (\{\max\{ \Delta_i, \exp(- \Psi_c) \}\}_{i\in S}, T)$ can be expressed as  
\begin{align*}
\frac{c}{2} x_i \max\{ \Delta_i, \exp(- \Psi_c)\}^2 + \ln \max\{ \Delta_i, \exp(- \Psi_c)\}^{-1}  = \Phi_{c},
\end{align*}
where 
\begin{align*}
\Phi_{c} = 
\argmax_x \{ x : \sum_{i = 1}^K \max\{  \frac{ x - \ln \max\{\Delta_i, \exp(- \Psi_c) \}^{-1} }{c \max\{\Delta_i, \exp(- \Psi_c) \}^2 / 2} , 0 \}  \leq T \}. 
\end{align*} 
It is easy to see that $\Phi_{c} = \Psi_c + c / 2$. Therefore the optimal solution to $\mathscr{P}_{c} (\{\max\{ \Delta_i, \exp(- \Psi_c) \}\}_{i\in S}, T)$ is
\begin{multline*}
x_i = \max\left\{ \frac{\Phi_{c} - \ln \max\{\Delta_i, \exp(- \Psi_c) \} ^{-1}}{ \frac{c}{2} \max\{\Delta_i, \exp(- \Psi_c)\}^2}, 0 \right\} 
\\ = \mathbb{I}_{\{  \Psi_c \leq \ln \Delta_i^{-1} \}} \cdot  \exp(2 \Psi_c ) + \mathbb{I}_{\{ \Psi_c > \ln \Delta_i^{-1} \}} \cdot \frac{\Psi_c - \ln \Delta_i^{-1} + c/2}{ c \Delta_i^2 / 2},
\end{multline*}
proving this lemma.
\end{proof}

Using Lemma~\ref{lem:solution-of-a-special-programming}, we derive the following useful inequality.

\begin{lemma} \label{lem:relationship-of-two-programmings}
Suppose $c > 0$ and let $\{x_i^*\}_{i \in S}$ be the solution to $\mathscr{P}_{c} (\{\max\{ \Delta_i, \exp(- \Psi_c) \}\}_{i\in S}, T)$. Then
\begin{align*}
\sum_{i \in S} \exp( -c x_i^* \Delta_i^2 ) 
  \leq \exp(c) \mathscr{P}_{c} (\{\max\{ \Delta_i, \exp(- \Psi_c) \}\}_{i\in S}, T) .
\end{align*}
\end{lemma}

\begin{proof}
By Lemma~\ref{lem:solution-of-a-special-programming}, the optimal solution to $\mathscr{P}_{c} (\{\max\{ \Delta_i, \exp(- \Psi_c) \}\}_{i\in S}, T)$ can be expressed as 

\begin{align}
x_i^* = \mathbb{I}_{\{  \Psi_c \leq \ln \Delta_i^{-1} \}} \cdot  \exp(2 \Psi_c ) + \mathbb{I}_{\{ \Psi_c > \ln \Delta_i^{-1} \}} \cdot \frac{\Psi_c - \ln \Delta_i^{-1} + c/2}{ c \Delta_i^2 / 2}.
\end{align}

Therefore, we obtain 
\begin{align*}
& \sum_{i \in S} \exp\left( - c x_i^* \Delta_i^2\right) \\
 \leq {} &  \exp(c)  \sum_{i \in S} \exp\left( - c x_i^* \max\{\Delta_i , \exp( - \Phi_c )\}^2 \right) \\
 = {} &  \exp(c) \mathscr{P}_c (\{\max\{ \Delta_i, \exp(- \Psi_c) \} \}_{i\in S}, T),
\end{align*}
and this lemma follows.
\end{proof}

Finally, we will show how the value of $\mathscr{P}_{c} ( \{\Delta_i\}_{i\in S}, T) $ will change when $c$ is changed.

\begin{lemma} \label{lem:change-c-of-a-programming}
If $c,c'>0$, then
\[
\mathscr{P}_{c} ( \{\Delta_i\}_{i\in S}, T) 
  =  \mathscr{P}_{c'} ( \{\Delta_i\}_{i\in S}, Tc / c').
\]
\end{lemma}

\begin{proof}
We observe that for any sequence of positive numbers $\{ x_i\}_{i \in S}$,
\[
\sum_{i = 1}^K \exp( - c x_i \Delta_i^2 ) = \sum_{i = 1}^K \exp( - c' \cdot (cx_i/c') \Delta_i^2 ).
\]
Suppose $\{ x_i \}_{i \in S}$ is the optimal solution to $\mathscr{P}_{c} ( \{\Delta_i\}_{i\in S}, T)$. Then $\{ cx_i/c' \}_{i \in S}$ is a feasible solution to $\mathscr{P}_{c'} ( \{\Delta_i\}_{i\in S}, Tc / c')$. Hence we obtain
$
\mathscr{P}_{c'} ( \{\Delta_i\}_{i\in S}, Tc / c') \leq \mathscr{P}_{c} ( \{\Delta_i\}_{i\in S}, T).
$
On the other hand, using a similar argument, we can also obtain
$
\mathscr{P}_{c} ( \{\Delta_i\}_{i\in S}, T) \leq \mathscr{P}_{c'} ( \{\Delta_i\}_{i\in S}, Tc / c').
$
Therefore, it holds that
\[
\mathscr{P}_{c} ( \{\Delta_i\}_{i\in S}, T) = \mathscr{P}_{c'} ( \{\Delta_i\}_{i\in S}, Tc / c'),
\]
and the lemma follows. 
\end{proof}

\section{Hard Instances for the Uniform Sampling Approach} \label{app:hard-instance-uniform-sampling}

In this section, we describe a class of bad instances for the uniform sampling approach. In such instances, we show that, to achieve the same order of regret, the uniform sampling approach needs at least $\Omega(K)$ times more budget than the optimal policy.

We fix the threshold $\theta = 0.5$. For each $K \geq 20$, we construct two instances $I_1$ and $I_2$. In $I_1$, we set $\theta_1 = 0.5 - \sqrt{1 / (K - 1)}$, and $\theta_{2:K} = 0.5 + \sqrt{0.1}$. In $I_2$, we set $\theta_1 = 0.5 + \sqrt{1 / (K - 1)}$, and $\theta_{2:K} = 0.5 + \sqrt{0.1}$. Hence for both instances, $\Delta_1 = \sqrt{1/(K - 1)}$ and $\Delta_{2:K} = \sqrt{0.1}$. Suppose $T = 2 K(K - 1) t_0$ where $t_0 \geq 10$. For simplicity, we use $\mathcal{R}^{\mathrm{uni}}(I; T)$ and $\mathcal{R}^{\mathrm{opt}}(I; T)$ to represent the regret incurred by the uniform sampling approach and the optimal policy on instance $I$ respectively. 

We now bound $\mathcal{R}^{\mathrm{uni}}(I; T)$ and $\mathcal{R}^{\mathrm{opt}}(I; T)$ in sequence. We first consider the uniform sampling approach and give a lower bound of regret incurred by it. Note that given $T = 2 K(K - 1) t_0$, the uniform sampling approach will play each arm $2(K - 1)t_0$ times. Let $\mathcal{K}$ denote the event that the classification for arm $1$ is incorrect. Define $\Pr_{I}[\cdot]$ as the probability induced by performing uniform sampling approach on instance $I$. We have
\begin{align}
\max\{ \mathcal{R}^{\mathrm{uni}}(I_1; T), \mathcal{R}^{\mathrm{uni}}(I_2; T) \} \notag & \geq \max\{ \Pr\nolimits_{I_1}( \mathcal{K} ), \Pr\nolimits_{I_2}( \mathcal{K} ) \} \notag \\ 
& \geq \frac{1}{2} \exp\left( - 16 \cdot \left( \sqrt{1 / (K - 1)} \right)^2 \cdot 2(K - 1)t_0 \right) \notag \\
& = \Omega( \exp( - 32t_0) ) \label{eq-1:app:example}
\end{align}
where the second inequality is obtained by applying Theorem~\ref{thm:lower-bound} when there is only one arm.

Next we derive an upper bound of regret incurred by the optimal policy. By setting $x_1 = K(K - 1) t_0 - (K - 1) \ln K $ and $x_{2:K} = Kt_0 + \ln K$, and using Chernoff-Hoeffding Inequality (Proposition~\ref{prop:chernoff-hoeff}), we have for any instance $I \in \{I_1, I_2\}$, it holds that
\begin{align}
&~~~ \mathcal{R}^{\mathrm{opt}}(I; T) \notag \\
& \leq 2 \exp\left( -2 \cdot \frac{1}{K - 1} \cdot (K(K - 1) t_0 - (K - 1)\ln K) \right) + 2(K - 1) \exp( - 2 \cdot 0.1 \cdot ( Kt_0 + \ln K )) \notag \\
& \leq 2(K + 1)^2 \exp( -0.2 K t_0) \label{eq-2:app:example}
\end{align}

For any $\eps \leq 1 / (K + 1)$, according to \eqref{eq-1:app:example}, there exists an instance $I' \in \{I_1, I_2\}$ such that, to achieve $\eps$ regret, the uniform sampling approach needs at least $\Omega( K^2 \ln \eps^{-1})$ budget. However, by \eqref{eq-2:app:example}, the optimal policy only needs at most $10 (K - 1) (\ln \frac{2}{\eps} + 2 \ln (K + 1)) \leq 20(K - 1) \ln \frac{2}{\eps} = O( K \ln \eps^{-1})$ plays for $I'$.

\section{Missing Proofs in Section~\ref{sec:upper-bound}} \label{app:proofs-ub}

\subsection{Proof of Theorem~\ref{thm:upper-bound}} \label{app:proof-main-thm-ub}

For convenience, we define the real-valued function $f(x) \eqdef \alpha x + \ln \alpha + 0.5 - \alpha$ and use $f^{-1}$ to denote its inverse. Also, we use ${\mathcal{R}_B^{\LSA}(T)}_{\cond \mathcal{F}}$ to denote the regret incurred by arms in $B$ when conditioned on event $\mathcal{F}$.

\begin{proof}[Proof of Theorem~\ref{thm:upper-bound}]

As discussed before, we only need to establish \eqref{eq:main-thm-1}, i.e.,
\[
    \mathcal{R}_B^{\LSA}(T)
    \leq \Phi(\alpha) \cdot \sum_{i \in B} \exp\left(\frac{ -  \lambda_i \Delta_i^2}{10} \right)  .
\]
Let $\Lambda' = \alpha \lfloor \frac{\Lambda}{\alpha} - 0.1 \rfloor$ and define the events $\mathcal{G}_0\eqdef\mathcal{F}_{\Lambda'}$, $\mathcal{G}_k\eqdef\bigwedge_{i=0}^{k-1}\overline{\mathcal{F}}_{\Lambda'- \alpha i} \land {\mathcal{F}}_{\Lambda'- \alpha k}$ if $1 \leq k \leq \lfloor \frac{\Lambda}{\alpha} -0.1 \rfloor$, and $\mathcal{G}_{\lfloor \frac{\Lambda}{\alpha} -0.1 \rfloor + 1} \eqdef\bigwedge_{i=0}^{\lfloor \frac{\Lambda}{\alpha} - 0.1\rfloor} \overline{\mathcal{F}}_{\alpha i}$. Note that the events $\mathcal{G}_0,\ldots,\mathcal{G}_{\lfloor \frac{\Lambda}{\alpha} - 0.1 \rfloor + 1}$ form a partition of the total probability space.  Then,
\begin{multline} \label{eq-1:thm:upper-bound}
  \mathcal{R}_B^{\LSA}(T) 
  = \sum_{k=0}^{\lfloor \frac{\Lambda}{\alpha} -0.1 \rfloor + 1} {\mathcal{R}_B^{\LSA}(T)}_{\cond \mathcal{G}_k} \cdot \Pr(\mathcal{G}_k)  \\
 \leq  {\mathcal{R}_B^{\LSA}(T)}_{\cond \mathcal{F}_{ \Lambda'} } + \sum_{k = 1}^{\lfloor \frac{\Lambda}{\alpha} -0.1 \rfloor}  {\mathcal{R}_B^{\LSA}(T)}_{\cond \mathcal{F}_{ \Lambda' - \alpha k} } \cdot \Pr(  \overline{\mathcal{F}}_{ \Lambda' - \alpha(k-1) }  ) + |B| \Pr( \overline{\mathcal{F}}_{0}   ) .
\end{multline}

Notice that that $\stat_{i, 10} \geq 0.5 \ln 10 > 0$. Moreover, since $T \geq 10K$, after $T$ rounds it holds that $\stat_{i}(T) > 0$ for all $i\in S$. Therefore, $\Pr( \overline{\mathcal{F}}_{0} ) = 0$. Recall that $f^{-1}(y) = \frac{ y + \alpha - \ln \alpha - 0.5 }{\alpha} $. Combining Lemma~\ref{lem:good-event-prob} and Lemma~\ref{lem:regret-incurred-by-bad-arm}, we upper bound $\eqref{eq-1:thm:upper-bound}$ by
\begin{align*}
& \sum_{i \in B} \frac{9 \cdot \sqrt[8\alpha]{2} }{ \sqrt[8\alpha]{2} - 1}  \exp\left( - \frac{ \lambda_i  \Delta_i^2}{10} + f^{-1}( 1.1\alpha) / 4 \right) \\
  & \qquad + \sum_{k = 1}^{\lfloor \frac{\Lambda}{\alpha} -0.1 \rfloor}  \sum_{i \in B} \frac{9 \cdot \sqrt[8\alpha]{2} }{ \sqrt[8\alpha]{2} - 1}  \exp \left( - \frac{\lambda_i \Delta_i^2}{10} + f^{-1}( (k+1.1)\alpha ) /4 \right) \exp( -40(k -0.9) )  \\
\leq {} & \frac{9 \cdot \sqrt[8\alpha]{2} }{ \sqrt[8\alpha]{2} - 1} \exp \left(  f^{-1}(1.1\alpha) / 4 \right) \Bigg( 1 + \sum_{k = 1}^{\lfloor \frac{\Lambda}{\alpha} - 0.1 \rfloor} \exp( -39.75k + 36 ) \Bigg) \cdot \sum_{i \in B} \exp\left( - \frac{ \lambda_i  \Delta_i^2}{10} \right)  \\
\leq {} & \frac{9.3 \cdot \sqrt[8\alpha]{2} }{ \sqrt[8\alpha]{2} - 1}  \exp \left( \frac{2.1 \alpha - \ln \alpha - 0.5}{ 4\alpha} \right)  \cdot \sum_{i \in B} \exp\left( - \frac{ \lambda_i  \Delta_i^2}{10} \right) 
\end{align*}
This completes the proof of (\ref{eq:main-thm-1}).
\end{proof}

\subsection{Proof of Lemma~\ref{lem:good-event-prob}} \label{sec:ub-step-1}

The goal of this subsection is to establish the following lemma which gives a lower bound on the probability of $\mathcal{F}_C$. 

\begin{namedthm*}{Lemma~\ref{lem:good-event-prob} (restated)}
$\Pr( \mathcal{F}_{\Lambda - k} ) \geq 1 - \exp(-40k/\alpha)$ for $0 \leq k < \Lambda $.
\end{namedthm*}

To prove Lemma~\ref{lem:good-event-prob}, we make use of Lemma~\ref{lem:expo-tail} and Lemma~\ref{lem:not-too-large-samples-prob}, and defer their proofs to the later part of this subsection.  

Recall that for any arm $i \in S$ and $C > 0$, $\tau_{i, C}$ is the random variable representing the smallest positive integer such that $\stat_{i, \tau_{i, C}} > C$. The following Lemma~\ref{lem:expo-tail} shows an exponentially small tail of the distribution of $\tau_{i, C}$. 

\begin{lemma} \label{lem:expo-tail}
For any arm $i \in S$, and $C > 0$, we have the following statements:
\begin{lemlist}
\item $\tau_{i, C} \leq 2 \exp(2C) $; \label{lem:expo-tail:a}
\item if $C > \ln \Delta_i^{-1}$, then for any $k \geq 1$, $\tau_{i, C}$ satisfies
\[
\Pr\left( \tau_{i, C} > \frac{40}{ \alpha } \cdot \frac{C - \ln \Delta_i^{-1} + k}{ \Delta_i^2 } \right) \leq 2\exp(- 40k/ \alpha ).
\]
\label{lem:expo-tail:b}
\end{lemlist} 
\end{lemma}

Based on Lemma~\ref{lem:expo-tail}, we are able to show that $\sum_{i=1}^{K} \tau_{i, C}$ also follows an exponential distribution, which leads to the following lemma.

\begin{lemma} \label{lem:not-too-large-samples-prob}
$\Pr( \sum_{i = 1}^K \tau_{i, \Lambda - k } \leq T ) \geq 1 - \exp(-40k/\alpha) $ for all $ 0 \leq k < \Lambda $.
\end{lemma}

 We are now ready to prove the main lemma (Lemma~\ref{lem:good-event-prob}) of this subsection.

\begin{proof}[Proof of Lemma~\ref{lem:good-event-prob}]
By Lemma~\ref{lem:not-too-large-samples-prob}, it suffices to prove that $\mathcal{F}_{\Lambda - k}$ occurs when $\sum_{i = 1}^K \tau_{i, \Lambda - k} \leq T$. So we assume that all the random rewards are generated before the algorithm starts and that $\sum_{i = 1}^K \tau_{i, \Lambda - k} \leq T$. 

Since $\stat_{i,t} \geq \ln t$, it is easy to see that there exists $T^*$ satisfying $\max_{i \in S} \stat_i(T^*) > \Lambda - k$ and $\max_{i \in S} \stat_i(t) \leq \Lambda - k$ for any  $1\leq t<T^*$.  We claim that $T^* \leq T$. Indeed, notice that for any arm $i \in S$ and $t \leq T^* - 1$, $\stat_i (t) \leq \Lambda  - k$. Hence $T_i(T^* - 1) < \tau_{i, \Lambda - k}$, and so $T^* - 1 = \sum_{i=1}^K T_i(T^* - 1) < \sum_{i=1}^K \tau_{i ,\Lambda - k} \leq T$. Therefore $T^*\leq T$.

Now, we assume without loss of generality that for arm $i^*\in S$, $\stat_{i^*}(T^*) > \Lambda - k$. Since at time $t$ Algorithm~\ref{alg:HTA} pulls $\argmin_{i \in S} \stat_{i}(t-1)$, arm $i^*$ will not be pulled until all the other arms $i\in S\setminus\{i^*\}$ satisfy $\stat_i(t-1) > \Lambda - k$. Since $\sum_{i = 1}^K \tau_{i, \Lambda - k} \leq T$, then we can find $T^{\natural}$ such that $T^* \leq T^{\natural} \leq T$ and $\stat_i(T^{\natural}) > \Lambda - k$ for any arm $i \in S$. This proves the lemma.
\end{proof}

\subsubsection{Proof of Lemma~\ref{lem:expo-tail}}


\begin{namedthm*}{Lemma~\ref{lem:expo-tail} (restated)}
For any arm $i \in S$, and $C > 0$, we have the following statements:
\begin{lemlist}
\item $\tau_{i, C} \leq 2 \exp(2C) $;
\item if $C > \ln \Delta_i^{-1}$, then for any $k \geq 1$, $\tau_{i, C}$ satisfies
\[
\Pr\left( \tau_{i, C} > \frac{40}{ \alpha } \cdot \frac{C - \ln \Delta_i^{-1} + k}{ \Delta_i^2 } \right) \leq 2\exp(- 40k/ \alpha ).
\]
\end{lemlist} 
\end{namedthm*}

\begin{proof}
We first prove Lemma~\ref{lem:expo-tail:a}. Note that if $t \geq \lfloor 2 \exp(2C) \rfloor$, then we have $\stat_{i, t} > 0.5\ln t \geq C$. Hence $t \leq \lfloor 2 \exp(2C) \rfloor \leq 2 \exp(2C)$ as desired. 

Now we prove Lemma~\ref{lem:expo-tail:b}. Note that $\forall k \geq 1$,
\begin{align*} 
 \Pr\left( \tau_{i, C} > \frac{40}{\alpha} \cdot \frac{C - \ln \Delta_i^{-1} + k}{ \Delta_i^2 } \right)  \leq {}   \Pr \left( \stat_{i, \tau_{i, C}} \leq C  \cond \tau_{i, C} = \left\lfloor \frac{40}{\alpha} \cdot \frac{C - \ln \Delta_i^{-1} + k}{ \Delta_i^2 } \right\rfloor \right).
\end{align*}

Assuming $\tau_{i, C} = \left \lfloor \frac{40}{\alpha} \cdot \frac{C - \ln \Delta_i^{-1} + k}{ \Delta_i^2 } \right \rfloor$ and $| \hat{\Delta}_{i, \tau_{i, C} } - \Delta_i|  < \sqrt{10} \Delta_i / 4$, we get that
\begin{align*}
 \stat_{i, \tau_{i, C}} 
= {} & \alpha \tau_{i, C} (\hat{\Delta}_{i, \tau_{i, C}})^2 + 0.5 \ln \tau_{i, C}  \\
> {} & \alpha \cdot \left \lfloor \frac{40}{\alpha} \cdot  \frac{C - \ln \Delta_i^{-1} + k}{ \Delta_i^2 } \right \rfloor \cdot (( 1- \sqrt{10} /4 ) \Delta_i)^2 +  0.5 \ln \tau_{i, C} \\
\geq {} & \alpha \cdot \frac{4}{5} \cdot \frac{40}{\alpha} \cdot  \frac{C - \ln \Delta_i^{-1} + k}{ \Delta_i^2 } \cdot (( 1- \sqrt{10} /4 ) \Delta_i)^2 +  0.5 \ln \tau_{i, C} \\
> {} & C - \ln \Delta_i^{-1} + k  + 0.5 \ln \tau_{i, C} 
> C,
\end{align*}
where we used $\tau_{i, C} = \left \lfloor \frac{40}{\alpha} \cdot  \frac{C - \ln \Delta_i^{-1} + k}{ \Delta_i^2 } \right \rfloor \geq \frac{4}{5} \cdot \frac{40}{\alpha} \cdot  \frac{C - \ln \Delta_i^{-1} + k}{ \Delta_i^2 } > \frac{4}{
\Delta_i^2} $ when $\alpha \leq 8$ and $k \geq 1$.

Therefore, we have that
\begin{align*}
& \Pr \left( \stat_{i, \tau_{i, C}} \leq C  \cond \tau_{i, C} = \left \lfloor \frac{40}{\alpha} \cdot \frac{C - \ln \Delta_i^{-1} + k}{ \Delta_i^2 } \right \rfloor \right) \\ 
\leq {} & \Pr\left( | \hat{\Delta}_{i, \tau_{i, C}} - \Delta_i | \geq \frac{\sqrt{10}}{4} \Delta_i \cond \tau_{i, C} = \left \lfloor \frac{40}{\alpha} \cdot \frac{C - \ln \Delta_i^{-1} + k}{ \Delta_i^2 } \right \rfloor \right) \\
\leq {} & \Pr\left( | \hat{\theta}_{i, \tau_{i, C}} - \theta_i | \geq \frac{\sqrt{10}}{4} \Delta_i \cond \tau_{i, C} = \left \lfloor \frac{40}{\alpha} \cdot \frac{C - \ln \Delta_i^{-1} + k}{ \Delta_i^2 } \right \rfloor\right) \\
\leq {} & 2 \exp\left( -2 \cdot \frac{4}{5} \cdot \frac{40}{\alpha} \cdot \frac{C - \ln \Delta_i^{-1} + k}{ \Delta_i^2 } \cdot ( \sqrt{10}\Delta_i/4 )^2 \right)\\ 
\leq {} &  2 \exp( -40k/\alpha),
\end{align*}
where the second inequality follows since $| \hat{\Delta}_{i, \tau_{i, C}} - \Delta_i | = | | \hat{\theta}_{i, \tau_{i, C}} - \tau | - |\theta_i - \tau| |\leq | \hat{\theta}_{i, \tau_{i, C}} - \theta_i |$, and the third inequality follows from Chernoff-Hoeffding Inequality (Proposition~\ref{prop:chernoff-hoeff}). This proves the desired result.
\end{proof}

\subsubsection{Proof of Lemma~\ref{lem:not-too-large-samples-prob}}

\begin{namedthm*}{Lemma~\ref{lem:not-too-large-samples-prob} (restated)} 
$\Pr( \sum_{i = 1}^K \tau_{i, \Lambda - k } \leq T ) \geq 1 - \exp(-40k/\alpha) $ for all $ 0 \leq k < \Lambda $.
\end{namedthm*}

\begin{proof}
Define the set $A\eqdef\{i\in S : \Lambda > \ln \Delta_i^{-1} + k\}$. We can assume without loss of generality that $A$ is not empty. Let $\mathcal{E}_1$ be the event
\begin{multline}
  \sum_{i\in S\setminus A} \tau_{i, \Lambda - k} \leq 
  \sum_{i\in S\setminus A} \Bigg ( \mathbb{I}_{\{  \Lambda  \leq \ln \Delta_i^{-1} \}} \cdot 2 \exp(2 \Lambda) 
   + \mathbb{I}_{\{ \Lambda > \ln \Delta_i^{-1} \}} \cdot \frac{40}{\alpha} \cdot \frac{\Lambda - \ln \Delta_i^{-1} + 1 + \alpha / 40}{ \Delta_i^2 } \Bigg) ;
\end{multline}
and let $\mathcal{E}_2$ be the event 
\begin{align}
 \sum_{i\in A} \tau_{i, \Lambda - k} \leq &  \sum_{i\in A}  \frac{40}{\alpha} \cdot \frac{\Lambda - \ln \Delta_i^{-1} + 1 + \alpha / 40}{ \Delta_i^2} .
\end{align}

Note that when $\mathcal{E}_1$ and $
\mathcal{E}_2$ hold, we have 
$\sum_{i = 1}^{K} \tau_{i, \Lambda - k}
\leq \sum_{i = 1}^{K} \max\{ 40/\alpha + 1, 40\} \lambda_i = T$.
Hence $\Pr( \sum_{i = 1}^{K} \tau_{i, \Lambda - k} > T) \leq \Pr( \overline{\mathcal{E}}_1 ) +  \Pr( \overline{\mathcal{E}}_2 )$, and since $\mathcal{E}_1$ always holds by Lemma~\ref{lem:expo-tail:a}, we have
\[
\Pr\left( \sum_{i = 1}^{K} \tau_{i, \Lambda - k} > T \right) 
\leq  \Pr\left(  \sum_{i\in A} \tau_{i, \Lambda - k} > \sum_{i\in A} \frac{40}{\alpha} \cdot \frac{\Lambda - \ln \Delta_i^{-1} + 1 + \alpha / 40}{ \Delta_i^2}  \right) .
\]

Now, for any arm $i \in A$, let $z_i = \tau_{i, \Lambda - k} - \frac{40}{\alpha} \cdot \frac{\Lambda - k - \ln \Delta_i^{-1} + 1}{\Delta_i^2} $. By Lemma~\ref{lem:expo-tail:b}, for any $x \geq 0$, $z_i$ satisfies
\[
\Pr( z_i > x) \leq 2\exp \left(- \frac{40}{\alpha} (\alpha x\Delta_i^2/40 + 1) \right) 
\leq \exp(-x \Delta_i^2) .
\]

Applying Proposition~\ref{prop:tail-of-sum-of-expo}, we have that for any $\lambda\geq 1$,
\begin{equation*} 
\Pr\left( \sum_{i \in A} z_i > \lambda \cdot \sum_{i \in A} \frac{1}{\Delta_i^2} \right) \leq \exp(1 - \lambda).
\end{equation*}
Therefore,
\begin{multline*}
 \Pr\left(  \sum_{i\in A } \tau_{i, \Lambda - k} > \sum_{i\in A} \frac{40}{\alpha} \cdot \frac{\Lambda - \ln \Delta_i^{-1} + 1 + \alpha / 40}{ \Delta_i^2}  \right) \\
 = \Pr\left( \sum_{i \in A} z_i > (40k/\alpha + 1) \cdot \sum_{i \in A} \frac{1}{\Delta_i^2} \right) 
\leq \exp(-40k/\alpha).
\end{multline*}
This completes the proof of the lemma.
\end{proof}

\subsection{Proof of Lemma~\ref{lem:regret-incurred-by-bad-arm}} \label{sec:ub-step-2}

Recall that we defined $B= \{ i\in S\ :\ \Lambda > \Delta_i^{-1} \}$ and $f(x) = \alpha x + \ln \alpha + 0.5 - \alpha$. We point out that if $x \geq \frac{ \alpha - \ln \alpha - 0.5 }{\alpha} + 0.1$, then $f(x) \geq 0.1\alpha$. The goal of this subsection is to build the following lemma.

\begin{namedthm*}{Lemma~\ref{lem:regret-incurred-by-bad-arm} (restated)}
If $\varkappa \geq  \frac{ \alpha - \ln \alpha - 0.5 }{\alpha} + 0.1$, then conditioned on $\mathcal{F}_{\Lambda - \num}$,
\[
\mathcal{R}_B^{\LSA}(T) \leq \frac{9 \cdot   \sqrt[8\alpha]{2} }{ \sqrt[8\alpha]{2} - 1} \cdot \sum_{i \in B}  \exp \left( - \frac{\lambda_i \Delta_i^2}{10}  + \varkappa /4 \right).
\]
\end{namedthm*}

To prove Lemma~\ref{lem:regret-incurred-by-bad-arm}, we make use of Lemma~\ref{lem:good-event-for-each-arm}, Lemma~\ref{lem:enough-pulls} and Corollary~\ref{corol:enough-pulls}, and defer their proofs in the later part of this subsection.

For any arm $i \in B$ and $\varkappa$, we define the event $\mathcal{M}_{i, \varkappa}$ by
\[
\mathcal{M}_{i, \varkappa} \eqdef \left\{ \forall t \in [1, \lambda_i ], |\hat{\Delta}_{i, t} - \Delta_i | \leq  \sqrt{ \frac{ \lambda_i \Delta_i^2 / 5 -  \varkappa /2 + \frac{1}{4\alpha} \ln \frac{\lambda_i}{t} }{t} } \right\}.
\]

Intuitively, $\mathcal{M}_{i, \varkappa}$ requires that the estimation error of $\Delta_i$ during any time of the algorithm stays within a small band that is parameterized by the quality parameter $\varkappa$. The following Lemma~\ref{lem:good-event-for-each-arm} gives a lower bound on $\mathcal{M}_{i, \varkappa}$.

\begin{lemma} \label{lem:good-event-for-each-arm} 
For any arm $i \in B$ and $\varkappa$, it holds that
$\Pr( \mathcal{M}_{i,\varkappa} ) \geq 1 - \frac{4 \cdot \sqrt[8\alpha]{2} }{ \sqrt[8\alpha]{2} - 1}  \exp \left( - \frac{\lambda_i \Delta_i^2}{10}  + \varkappa / 4 \right) $. 
\end{lemma}

The following Lemma~\ref{lem:enough-pulls} shows that $\mathcal{M}_{i, \varkappa}$ together with $\mathcal{F}_{\Lambda - \num}$ guarantees that arm $i$ is explored by enough queries.

\begin{lemma} \label{lem:enough-pulls}
For any arm $i \in B$ and $\varkappa \geq \frac{ \alpha - \ln \alpha - 0.5 }{\alpha}$, conditioning on $\mathcal{M}_{i, \varkappa} \wedge \mathcal{F}_{\Lambda - \num}$, we have that $T_i(T) \geq \lambda_i / 20$.
\end{lemma}

A corollary of  Lemma~\ref{lem:enough-pulls} is as follows.

\begin{corollary} \label{corol:enough-pulls}
For any arm $i \in B$ and $\varkappa \geq \frac{ \alpha - \ln \alpha - 0.5 }{\alpha}$, we have 
\[
\Pr(T_i(T) < \lambda_i / 20 \cond \mathcal{F}_{\Lambda - \num} ) \leq \frac{ \Pr( \overline{ \mathcal{M}}_{i, \varkappa} ) }{ \Pr( \mathcal{F}_{\Lambda - \num} ) }.
\]
\end{corollary}

We are now ready to give an upper bound for the contribution of the arms in $B$ to the aggregate regret of Algorithm~\ref{alg:HTA}. 

\begin{proof}
Let $i$ be an arbitrary arm in $B$. Since $\Pr( \overline{\mathcal{E}}_i(T) \cond  \mathcal{F}_{\Lambda - \num} ) \leq 1$, it suffices to prove that 
\begin{equation} \label{eq:regret-incurred-by-bad-arm-1}
\Pr( \overline{\mathcal{E}}_i(T) \cond  \mathcal{F}_{\Lambda - \num} ) 
\leq \frac{9 \cdot \sqrt[8\alpha]{2} }{ \sqrt[8\alpha]{2} - 1} \exp \left( - \frac{\lambda_i \Delta_i^2}{10}  + \varkappa / 4 \right)
\end{equation} 
whenever $ \frac{\sqrt[8\alpha]{2} }{ \sqrt[8\alpha]{2} - 1}  \exp \left( - \frac{\lambda_i \Delta_i^2}{10}  + \varkappa /4 \right)  \leq 1 / 9$. Then the lemma follows by summing up the inequality for all arms in $B$. 

Notice that 
\begin{align}
& \Pr( \overline{\mathcal{E}}_i(T) \cond  \mathcal{F}_{\Lambda - \num} )  \nonumber \\
= {} & \Pr( \overline{\mathcal{E}}_i(T) \cond  T_i \geq \lambda_i / 20, \mathcal{F}_{\Lambda - \num} ) \Pr( T_i \geq \lambda_i / 20 \cond  \mathcal{F}_{\Lambda - \num} ) \nonumber \\
& + \Pr( \overline{\mathcal{E}}_i(T) \cond  T_i <  \lambda_i / 20, \mathcal{F}_{\Lambda - \num} ) \Pr( T_i < \lambda_i / 20 \cond  \mathcal{F}_{\Lambda - \num} )  \nonumber \\
\leq {} & \Pr( \overline{\mathcal{E}}_i(T) \cond  T_i \geq \lambda_i / 20, \mathcal{F}_{\Lambda - \num} )  + \Pr( T_i < \lambda_i / 20 \cond  \mathcal{F}_{\Lambda - \num} ).  \label{eq-0:lem:regret-incurred-by-good-arm}
\end{align}
We first focus on the first term of \eqref{eq-0:lem:regret-incurred-by-good-arm}, and note that
\begin{align}
    &\Pr( \overline{\mathcal{E}}_i(T) \cond  T_i \geq \lambda_i / 20, \mathcal{F}_{\Lambda - \num} ) \nonumber \\
    = {} & \frac{ \Pr( \overline{\mathcal{E}}_i(T) \wedge \mathcal{F}_{\Lambda - \num} \cond  T_i \geq \lambda_i / 20 ) }{\Pr( \mathcal{F}_{\Lambda - \num} \cond  T_i \geq \lambda_i / 20) } \nonumber  \\
     \leq {} & \frac{ \Pr( \overline{\mathcal{E}}_i(T) \cond  T_i \geq \lambda_i / 20 ) }{\Pr( \mathcal{F}_{\Lambda - \num } \cond  T_i \geq \lambda_i / 20) } = \frac{ \Pr( \overline{\mathcal{E}}_i(T) \wedge  T_i \geq \lambda_i / 20 ) }{\Pr( \mathcal{F}_{\Lambda - \num} \wedge  T_i \geq \lambda_i / 20) } \nonumber \\
     = {} &\frac{ \Pr( \overline{\mathcal{E}}_i(T) \wedge  T_i \geq \lambda_i / 20 ) }{ (1 - \Pr( T_i < \lambda_i / 20 \cond \mathcal{F}_{\Lambda - \num} )) \cdot \Pr( \mathcal{F}_{\Lambda - \num} ) }. \label{eq-0a:lem:regret-incurred-by-good-arm}
\end{align}
Then plugging \eqref{eq-0a:lem:regret-incurred-by-good-arm} into \eqref{eq-0:lem:regret-incurred-by-good-arm}, we derive
\begin{equation} \label{eq-0b:lem:regret-incurred-by-good-arm}
\Pr( \overline{\mathcal{E}}_i(T) \cond  \mathcal{F}_{\Lambda - \num} ) \leq \frac{ \Pr( \overline{\mathcal{E}}_i(T) \wedge  T_i \geq \lambda_i / 20 ) }{ (1 - \Pr( T_i < \lambda_i / 20 \cond \mathcal{F}_{\Lambda - \num} )) \cdot \Pr( \mathcal{F}_{\Lambda - \num} ) } + \Pr( T_i < \lambda_i / 20 \cond  \mathcal{F}_{\Lambda - \num} ).
\end{equation}

Using Chernoff-Hoeffding Inequality (Proposition~\ref{prop:chernoff-hoeff}), we have
\begin{multline} \label{eq-2:lem:regret-incurred-by-good-arm}
\Pr( \overline{\mathcal{E}}_i(T) \wedge T_i \geq \lambda_i / 20 ) = \sum_{t = \lceil \lambda_i / 20 \rceil }^ {+\infty} \Pr( \overline{\mathcal{E}}_i(T) \cond T_i = t ) \Pr( T_i = t) \\
\leq \sum_{t = \lceil \lambda_i / 20 \rceil }^ {+\infty} \Pr(T_i = t) \cdot 2 \exp( - \lambda_i  \Delta_i^2 / 10) \leq 2 \exp( - \lambda_i  \Delta_i^2 / 10).
\end{multline}
Moreover, by Lemma~\ref{lem:good-event-prob} and the fact that $f(\varkappa) \geq 0.1\alpha$ for $\varkappa \geq \frac{ \alpha - \ln \alpha - 0.5 }{\alpha} + 0.1$,  we have
\begin{equation} \label{eq-3:lem:regret-incurred-by-good-arm}
\Pr( \mathcal{F}_{\Lambda - \num} ) \geq 1 - \exp(- 40 f(k) / \alpha ) \\
\geq 1 - \exp(-4) \geq 0.9.
\end{equation} 
Combining (\ref{eq-3:lem:regret-incurred-by-good-arm}) with Corollary~\ref{corol:enough-pulls} and Lemma~\ref{lem:good-event-for-each-arm}, we have
\begin{align}  \label{eq-4:lem:regret-incurred-by-good-arm}
&\Pr( T_i < \lambda_i / 20 \cond  \mathcal{F}_{\Lambda - \num} ) \notag\\
\leq {} & \frac{ \Pr( \overline{ \mathcal{M}}_{i, \varkappa} ) }{ \Pr( \mathcal{F}_{\Lambda - \num} ) } \leq \frac{ \frac{4 \cdot \sqrt[8\alpha]{2} }{ \sqrt[8\alpha]{2} - 1}  \exp \left( - \frac{\lambda_i \Delta_i^2}{10}  + \varkappa/4 \right)  }{ 0.9 } \notag\\
\leq {} & \frac{4.5 \cdot \sqrt[8\alpha]{2} }{ \sqrt[8\alpha]{2} - 1}  \exp \left( - \frac{\lambda_i \Delta_i^2}{10}  + \varkappa / 4 \right) .
\end{align}
Putting together 
\eqref{eq-0b:lem:regret-incurred-by-good-arm}, \eqref{eq-2:lem:regret-incurred-by-good-arm}, \eqref{eq-3:lem:regret-incurred-by-good-arm}, and \eqref{eq-4:lem:regret-incurred-by-good-arm}, we obtain
\begin{align*}
&\Pr( \overline{\mathcal{E}}_i(T) \cond  \mathcal{F}_{\Lambda - \num} ) \\
\leq {} & \frac{  2 \exp( - \lambda_i  \Delta_i^2 / 10) }{ \left(1 - \frac{4.5 \cdot \sqrt[8\alpha]{2} }{ \sqrt[8\alpha]{2} - 1}  \exp \left( - \frac{\lambda_i \Delta_i^2}{10}  + \varkappa / 4 \right)  \right) \cdot 0.9 } + \frac{4.5 \cdot \sqrt[8\alpha]{2} }{ \sqrt[8\alpha]{2} - 1}  \exp \left( - \frac{\lambda_i \Delta_i^2}{10}  + \varkappa / 4 \right)  \\
\leq {} & \left( \frac{2}{(1 - 4.5 /9) \cdot 0.9} + \frac{ 4.5 \cdot \sqrt[8\alpha]{2} }{ \sqrt[8\alpha]{2} - 1}  \right) \exp \left( - \frac{\lambda_i \Delta_i^2}{10}  + \varkappa / 4 \right)  \\
\leq {} & \frac{9 \cdot   \sqrt[8\alpha]{2} }{ \sqrt[8\alpha]{2} - 1}  \exp \left( - \frac{\lambda_i \Delta_i^2}{10}  + \varkappa / 4 \right) ,
\end{align*}
where the second inequality follows from our assumption that $ \frac{ \sqrt[8\alpha]{2} }{ \sqrt[8\alpha]{2} - 1}  \exp \left( - \frac{\lambda_i \Delta_i^2}{10}  + \varkappa / 4 \right)  \leq 1 / 9$. This proves \eqref{eq:regret-incurred-by-bad-arm-1} and therefore the lemma.
\end{proof}

\subsubsection{Proof of Lemma~\ref{lem:good-event-for-each-arm}}\label{sec:proof-ub-step2-1}

\begin{namedthm*}{Lemma~\ref{lem:good-event-for-each-arm} (restated)}
For any arm $i \in B$ and $\varkappa$, it holds that
$\Pr( \mathcal{M}_{i,\varkappa} ) \geq 1 - \frac{4 \cdot \sqrt[8\alpha]{2} }{ \sqrt[8\alpha]{2} - 1}  \exp \left( - \frac{\lambda_i \Delta_i^2}{10}  + \varkappa /4 \right) $. 
\end{namedthm*}

In order to estimate the probability of $\mathcal{M}_{i, \varkappa}$, we introduce a more general lemma as follows and Lemma~\ref{lem:good-event-for-each-arm} becomes a simple corollary of Lemma~\ref{lem:variational-lil}.

\begin{lemma}{(Variable Confidence Level Bound)} \label{lem:variational-lil}
Let $X_1, \dots, X_L$ be \textit{i.i.d.}\ random variables supported on $[0, 1]$ with mean $\mu$. For any $a > 0$ and $b > 0$, it holds that
\[
\Pr\left( \forall t \in [1, L], \left| \frac{1}{t} \sum_{i = 1}^{t} X_i - \mu \right| \leq  \sqrt{ \frac{ a + b \ln (L/t)}{t} } \right) \geq 1 - \frac{2^{b/2+2} }{ 2^{b/2} - 1}  \exp (- a/2).
\] 
\end{lemma}

Now we only need to prove Lemma~\ref{lem:variational-lil}.

\begin{proof}[Proof of Lemma~\ref{lem:variational-lil}]
Let $l = \lfloor \log_2 L \rfloor$. By Chernoff-Hoeffding Inequality (Proposition~\ref{prop:chernoff-hoeff}), we have for any $t \in \{1, 2, 4, \dots, 2^{l} \}$, 
\[
\Pr \left( \left| \frac{1}{t} \sum_{i = 1}^t X_i - \mu \right| \leq  \frac{1}{2} \sqrt{ \frac{ a + b \ln \frac{L}{t} }{t} } \right) \geq 1 - 2  \exp ( -a/2) \cdot \frac{t^{b/2}}{L^{b/2}}.
\]
Via a union bound and using the fact that $2^{l+1} \leq 2 L$, we get 
\begin{multline}  \label{eq-1:lem:variational-lil}
\Pr \left( \forall t \in \{1, 2, 4, \dots, 2^{l} \}, \left| \frac{1}{t} \sum_{i = 1}^t X_i - \mu \right| \leq  \frac{1}{2} \sqrt{ \frac{ a + b \ln \frac{L}{t} }{t} } \right) \\
\geq 1 - 2  \exp ( -a/2) \cdot \sum_{i = 0}^l \frac{2^{bi/2}}{L^{b/2}} \geq 1 - \frac{2^{b/2 + 1} }{ 2^{b/2} - 1}  \exp ( -a/2) .
\end{multline}

By Hoeffding's Maximal Inequality (Proposition~\ref{prop:hoeff-max-inequ}), we have for any $t \in \{1, 2, 4, \dots, 2^{l} \}$,
\begin{multline*}
\Pr \left( \forall j \in [1, \min\{t, L - t\}], | X_{i, t + 1} + \cdots + X_{i, t + j} - j \mu | \leq  \frac{1}{2} \sqrt{ t \left( a + b \ln \frac{L}{t} \right) } \right) \\ \geq 1 - 2  \exp ( - a/2 ) \cdot \frac{t^{b/2}}{L^{b/2}}.
\end{multline*}
Again via a union bound and using the fact that $2^{l+1} \leq 2 L$, we get 
\begin{multline}  \label{eq-2:lem:variational-lil}
\Pr \Bigg( \forall t \in \{1, 2, 4, \dots, 2^{l} \}, \forall j \in [1, \min\{t, L - t\}], | X_{i, t + 1} + \cdots + X_{i, t + j} -  j \mu | \leq  \frac{1}{2} \sqrt{ t \left( a + b \ln \frac{L}{t} \right) } \Bigg) \\
 \geq 1 - \frac{2^{b/2+1} }{ 2^{b/2} - 1}  \exp (-a/2) .
\end{multline}

Combining \eqref{eq-1:lem:variational-lil} and \eqref{eq-2:lem:variational-lil}, and using a union bound, we have with probability at least $1 - \frac{2^{b/2+2} }{ 2^{b/2} - 1}  \exp (- a/2)$ uniformly over all $t \in \{1, 2, 4, \dots, 2^{l} \}$ and $ j \in [1, \min\{t, \lambda_i - t\}]$ that
\[
| X_{1} + \cdots + X_{t + j} -  (t + j) \mu | \leq \sqrt{ t \left( a + b \ln \frac{L}{t} \right) }.
\]
Dividing both sides of the above inequality by $(t + j)$, we complete the proof of this lemma.

\end{proof}

\subsubsection{Proof of Lemma~\ref{lem:enough-pulls} and Corollary~\ref{corol:enough-pulls}}

\begin{namedthm*}{Lemma~\ref{lem:enough-pulls} (restated)}
For any arm $i \in B$ and $\varkappa \geq \frac{ \alpha - \ln \alpha - 0.5 }{\alpha}$, conditioning on $\mathcal{M}_{i, \varkappa} \wedge \mathcal{F}_{\Lambda - \num}$, we have that $T_i(T) \geq \lambda_i / 20$.
\end{namedthm*}

\begin{proof}
Fix an arm $i \in B$ and $\varkappa \geq \frac{ \alpha - \ln \alpha - 0.5}{\alpha}$. We now condition on $\mathcal{M}_{i, \varkappa} \wedge \mathcal{F}_{\Lambda - \num}$ and prove this lemma by contradiction. Suppose for contradiction that we have $t < \lambda_i / 20$.
Notice that
 \begin{align}
     \stat_{i, t} \nonumber 
    \leq {}&  \alpha  t\left(  \Delta_i  + \sqrt{ \frac{ \lambda_i \Delta_i^2 / 5 -  \varkappa / 2 + \frac{1}{4\alpha} \ln \frac{\lambda_i}{t} }{t} }  \right)^2 + \ln \sqrt{t} \notag \\
     \leq{} &  \alpha \left(  \sqrt{\frac{\lambda_i}{20}}\Delta_i  + \sqrt{ \lambda_i \Delta_i^2 / 5 -  \varkappa / 2 + \frac{1}{4\alpha} \ln \frac{\lambda_i}{t}  } \right)^2 + \ln \sqrt{t} \notag\\
     \leq{} & \alpha \left( 0.5 \lambda_i\Delta_i^2 - \varkappa + \frac{1}{2\alpha} \ln \frac{\lambda_i}{t} \right) + \ln \sqrt{t} \notag\\
     \leq{} & \alpha ( 0.5\lambda_i\Delta_i^2 - \varkappa)+\ln\sqrt{\lambda_i}  \label{eq:lem-enough pulls},
 \end{align} 
 It is easy to verify that $x= (\alpha \lambda_i)^{-\frac{1}{2}}$ is the minimum of the function $0.5 \alpha \lambda_i x^2+\ln x^{-1}$ when $x > 0$. Hence, we have
 \begin{align*}
 & \ln \sqrt{\lambda_i} + \ln \sqrt{\alpha} + 0.5 
 = {}  0.5\alpha \lambda_i  ( \alpha \lambda_i)^{-1} + \ln  ( \alpha \lambda_i)^{\frac{1}{2}} 
     \leq {}   0.5\alpha \lambda_i \Delta_i^2 + \ln \Delta_{i}^{-1}.
 \end{align*}
 Therefore,
 \begin{align*}
     (\ref{eq:lem-enough pulls}) \leq {} &  \alpha \lambda_i\Delta_i^2 + \ln\Delta_i^{-1}- \alpha - ( \alpha \varkappa + \ln \alpha + 0.5 - \alpha)  .
 \end{align*} 
 Finally, since $\Lambda= \alpha \lambda_i\Delta_i^2 + \ln\Delta_i^{-1} - \alpha $ for all $i\in B$ by definition, the last inequality yields $\stat_{i,t}\leq \Lambda- \num $, which contradicts the assumption that $\mathcal{F}_{\Lambda - \num}$ is true.
\end{proof}

\begin{namedthm*}{Corollary~\ref{corol:enough-pulls} (restated)}
For any arm $i \in B$ and $\varkappa \geq \frac{ \alpha - \ln \alpha - 0.5 }{\alpha}$, we have 
\[
\Pr(T_i(T) < \lambda_i / 20 \cond \mathcal{F}_{\Lambda - \num} ) \leq \frac{ \Pr( \overline{ \mathcal{M} }_{i, \varkappa} ) }{ \Pr( \mathcal{F}_{\Lambda - \num} ) }.
\]
\end{namedthm*}

\begin{proof}
Note that
\begin{align*}
    \Pr(T_i(T) < \lambda_i / 20 \cond \mathcal{F}_{\Lambda - \num} ) 
    = 1 - \frac{ \Pr( T_i(T) \geq \lambda_i / 20 \wedge \mathcal{F}_{\Lambda - \num} ) }{\Pr( \mathcal{F}_{\Lambda - \num} )}.
\end{align*}
Further by Lemma~\ref{lem:enough-pulls}, we obtain
\begin{align*}
\Pr(T_i(T) < \lambda_i / 20 \cond \mathcal{F}_{\Lambda - \num} )
\leq  1 - \frac{ \Pr( \mathcal{M}_{i, \varkappa} \wedge \mathcal{F}_{\Lambda - \num} ) }{\Pr( \mathcal{F}_{\Lambda - \num} )} 
= \frac{ \Pr( \overline{\mathcal{M}}_{i, \varkappa}  \wedge \mathcal{F}_{\Lambda - \num} ) }{ \Pr(\mathcal{F}_{\Lambda - \num}) }
\leq  \frac{ \Pr( \overline{ \mathcal{M}}_{i, \varkappa} ) }{ \Pr( \mathcal{F}_{\Lambda - \num} ) },
\end{align*}
which concludes the proof of this corollary.
\end{proof}

\section{The Lower Bound} \label{sec:lower-bound}

In this section we discuss the regret lower bound of \emph{any} algorithm. For any sequence of $K$ gaps $\Delta_1, \ldots, \Delta_K > 0$, let $\mathcal{I}_{\Delta_1, \dots, \Delta_K}$ denote the set of instances of the problem where the gap between $\theta_i$ and $\theta$ is $\Delta_i$ for every arm $i \in [K]$.   We now show a parameter dependent lower bound of the aggregate regret when the time horizon $T \geq K$.

\begin{theorem} \label{thm:lower-bound}
Let $(\Delta_1, \dots, \Delta_K) \in (0, 1/4]^K$ be a sequence of gaps. Then for any algorithm $\mathbb{A}$ and time horizon $T \geq K$, there exists an instance $I \in \mathcal{I}_{\Delta_1, \dots, \Delta_K}$ such that
 \[
 \mathcal{R}^{\mathbb{A}}(I; T) \geq \frac{1}{4} \mathscr{P}_{16} (\{\Delta_i\}_{i\in S}, T).
 \]
\end{theorem}

\begin{proof}
Let $\mathcal{B}(\mu)$ denote the Bernoulli distribution with mean $\mu$. We use $\Delta_i(I)$ to denote the gap between arm $i$ and the threshold $\theta$, given the instance $I$. For any algorithm $\mathbb{A}$ and instance $I$, let $\mathcal{D}_{I\mathbb{A}}$ denote the probability space induced by $I$ and $\mathbb{A}$. We use $\Pr_{I \mathbb{A}}$ to denote the measure of the probability space $\mathcal{D}_{I\mathbb{A}}$, and use $\E_{I\mathbb{A}}[ \cdot ]$ to denote the expectation with respect to $\Pr_{I \mathbb{A} }$. When clear from the context, the reference to $\mathbb{A}$ is omitted.

We fix the threshold $\theta = 1/2$.
To prove the theorem, it suffices to prove that there exist $2^K$ instances $I_0, \dots, I_{2^K - 1} \in \mathcal{I}_{\Delta_1, \dots, \Delta_K}$ such that 
\[
\max_{0 \leq j < 2^K} \mathcal{R}^{\mathbb{A}}(I_j; T) \geq \frac{1}{4} \mathscr{P}_{16} (\{\Delta_i\}_{i\in S}, T).
\]
We now define these instances explicitly. Suppose the binary representation of $j$ is denoted by $\overline{a^j_1\cdots{}a^j_K}$. Then for any arm $i$ in $I_j$, the associated distribution is $\mathcal{B}(1/2 + a^j_i \Delta_i)$. Thus the distribution associated with $I_j$ can be represented by the product distribution
\[
\mathcal{B}(1/2 + a^j_1 \Delta_1) \otimes \cdots \otimes \mathcal{B}(1/2 + a^j_K \Delta_K).
\]  
    
First, we note that 
\begin{align} \label{eq-0:thm:lower-bound}
  \max_{0 \leq j < 2^K} \mathcal{R}^{\mathbb{A}}(I_j; T) &= \max_{0 \leq j < 2^K} \sum_{i = 1}^K \Pr\nolimits_{I_j} (  \overline{\mathcal{E}}_i(T) ) \notag\\
  &\geq \frac{1}{2^K} \sum_{j = 0}^{2^K - 1} \sum_{i = 1}^K \Pr\nolimits_{I_j} (  \overline{\mathcal{E}}_{i}(T) ).
\end{align}
By counting $\Pr\nolimits_{I_j} (  \overline{\mathcal{E}}_{i}(T) )$ twice and then reordering, we have 
\begin{equation} \label{eq-1:thm:thm:lower-bound}
  \eqref{eq-0:thm:lower-bound} =  \frac{1}{2^{K + 1} } \sum_{i = 1}^K \sum_{j = 0}^{2^K - 1}   \left( \Pr\nolimits_{I_j} (  \overline{\mathcal{E}}_{i}(T) ) + \Pr\nolimits_{I_{j \oplus 2^{i - 1}} } (  \overline{\mathcal{E}}_{i}(T) )  \right)  ,
\end{equation}
where $\oplus$ denotes the binary XOR operation. Now from Proposition~\ref{prop:lower-bound-technique}, we get that for $i\in[K]$,
\begin{align}\label{eq-2:thm:thm:lower-bound}
 & \Pr\nolimits_{I_j} (  \overline{\mathcal{E}}_{i}(T) ) + \Pr\nolimits_{I_{j \oplus 2^{i - 1}} } (  \overline{\mathcal{E}}_{i}(T) )  \notag\\
 \geq {} &  \frac{1}{2} \exp( -\KL( \Pr\nolimits_{I_j} \parallel \Pr\nolimits_{ I_{j \oplus 2^{i - 1}} } ) ) \notag \\
  \overset{(a)}{=} {} &  \frac{1}{2} \exp( -  \E\nolimits_{I_j}[x_i] \cdot \KL( \mathcal{B}(1/2 + a^j_i \Delta_i ) \parallel  \mathcal{B}(1/2 + a^{j \oplus 2^{i - 1} }_i \Delta_i ) ) \nonumber \\
  = {} & \frac{1}{2} \exp \left( - \E\nolimits_{I_j}[x_i] \cdot 2\Delta_i \ln \left( 1 + \frac{ 2\Delta_i}{ 1/2 - \Delta_i } \right) \right) \notag \\
  \overset{(b)}{\geq} {}  & \frac{1}{2} \exp \left( - \frac{ 4 \E\nolimits_{I_j}[x_i]\Delta_i^2}{ 1/2 - \Delta_i }  \right) \notag \\
  \geq {}  & \frac{1}{2} \exp( - 16  \E\nolimits_{I_j}[x_i] \Delta_i^2 ), 
\end{align}
where (a) follows from standard divergence decomposition (Proposition~\ref{prop:divergence-decomposition}) and (b) follows from $\Delta_i \leq 1/4$.
Finally, plugging \eqref{eq-2:thm:thm:lower-bound} into \eqref{eq-1:thm:thm:lower-bound}, we have 
\begin{align*}
  &\max_{0 \leq j < 2^K} \mathcal{R}^{\mathbb{A}}(I_j; T) & \\
  \geq {} & \frac{1}{2^{K + 1} } \sum_{j = 0}^{2^K - 1} \sum_{i = 1}^K \frac{1}{2} \exp( - 16 \E\nolimits_{I_j} [x_i] \Delta_i^2 ) \\
  \geq {} &  \frac{1}{2^{K + 1} } \sum_{j = 0}^{2^K - 1} \min_{ \substack{x_1 + \cdots + x_K = T \\ x_1,\ldots,x_K\geq0 } } \sum_{i = 1}^K \frac{1}{2} \exp( - 16x_i\Delta_i^2 ) \\
  = {} &  \min_{ \substack{x_1 + \cdots + x_K = T \\ x_1,\ldots,x_K\geq0 } } \sum_{i = 1}^K \frac{1}{4} \exp( - 16x_i\Delta_i^2 ) \\
  ={} & \frac{1}{4} \mathscr{P}_{16}(\{\Delta_i\}_{i\in S}, T),
\end{align*}
which concludes the proof of this theorem.
\end{proof}

\section{Additional Experimental Results} \label{app:more-experiments}

\subsection{T-tests}\label{app:more-experiments-c}

In order to statistically compare our algorithm and other algorithms, we perform two-tailed paired t-tests between our algorithm and other algorithms respectively on $5000$ independent runs. When performing t-tests, we set $T = 1000$, $T = 40000$, and $T = 100000$ in Setup 1, 2 and 3 respectively. The null hypothesis is that our algorithm and other algorithms have the same mean. The p-values are listed in the following table.

\begin{small}
\begin{table}[H]
    \centering
    \begin{small}
    \begin{tabular}{cccccc}
    \toprule
    \multirow{2}{*}{Setup} & \multicolumn{5}{c}{P-values} \\
    \cline{2-6}
    &  APT(0) & APT(.025) & APT(.05) & APT(.1) & UCBE(-1)  \\
         \hline
      Setup 1 & 1.4e-32 & 0.60 & 0.95 & 1.3e-5 & 0  \\
      Setup 2 & 0 & 8.9e-21 & 1.5e-78 & 4.0e-160 & 0  \\
      Setup 3 & 0 & 7.8e-28 & 3.6e-88 & 1.5e-192 & 0  \\
      \midrule
    \multirow{2}{*}{Setup} & \multicolumn{5}{c}{P-values} \\
    \cline{2-6}
    &   UCBE(0) & UCBE(4) & Opt-KG(1,1) & Opt-KG(.5,.5) & Uniform \\
         \hline
      Setup 1 & 2.5e-69 & 4.1e-225 & 7.8e-3 & 5.6e-8 & 1.1e-295 \\
      Setup 2  & 0 & 1.6e-251 & 2.9e-26 & 4.6e-46 & 0 \\
      Setup 3 & 0 & 2.0e-157 & 2.1e-24 & 6.5e-36 & 0 \\
      \bottomrule
    \end{tabular}
    \caption{T-test results between our algorithm and other algorithms in Setup 1, 2 and 3}
    \label{tab:p-values}
    \end{small}
\end{table}
\end{small}

\subsection{Experimental Results for More Setups}\label{app:more-experiments-a}


We present additional experimental results with other settings of $\{ \theta_i \}_{i=1}^K$.

\begin{enumerate}
\item[4.] (geometric progression). $K = 10$; $\theta_{1:4} = 0.4 - 0.2^{(1:4)}, \theta_5 = 0.45, \theta_6 = 0.55$, and $\theta_{7:10} = 0.6 + 0.2^{5 - (1:4)}$ (see Setup 4 in Figure~\ref{fig:exp2}).
\item[5.](two-group setting II). $K = 100$; $\theta_{1:50} = 0.4$ and $\theta_{51:100} = 0.51$ (see Setup 5 in Figure~\ref{fig:exp2}).
\item[6.] (one-side group). $K = 10$; $\theta_{1:10} = 0.4 + (i - 1) / 100$ (see Setup 6 in Figure~\ref{fig:exp2}).
\end{enumerate}

\begin{figure}[H]
\centering
 \includegraphics[width=\textwidth]{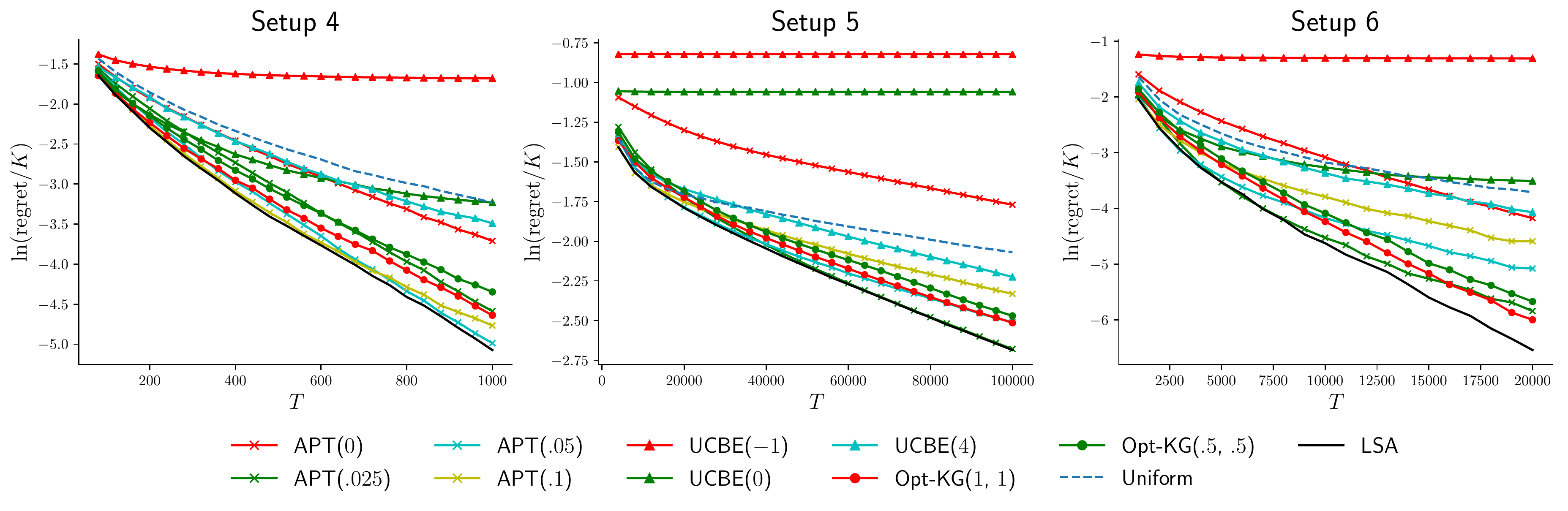}
\caption{Average aggregate regret on a logarithmic scale for additional set of settings.}
\label{fig:exp2}
\end{figure}

\subsection{Robustness of the Tuning Parameter in \LSA}\label{app:more-experiments-b}
To test the robustness of our algorithm, we show the results of our algorithm for Setup 1 when varying $\alpha$ in Figure~\ref{fig:exp3}. We find that the performance is very consistent with different choices of $\alpha$. The differences are marginal and not statistically significant. For simplicity, in our experiments in the main text, we use $\alpha = 1.35$ as default (black curve in the figure). 

\begin{figure}[H]
\centering
\includegraphics[width=0.5\textwidth]{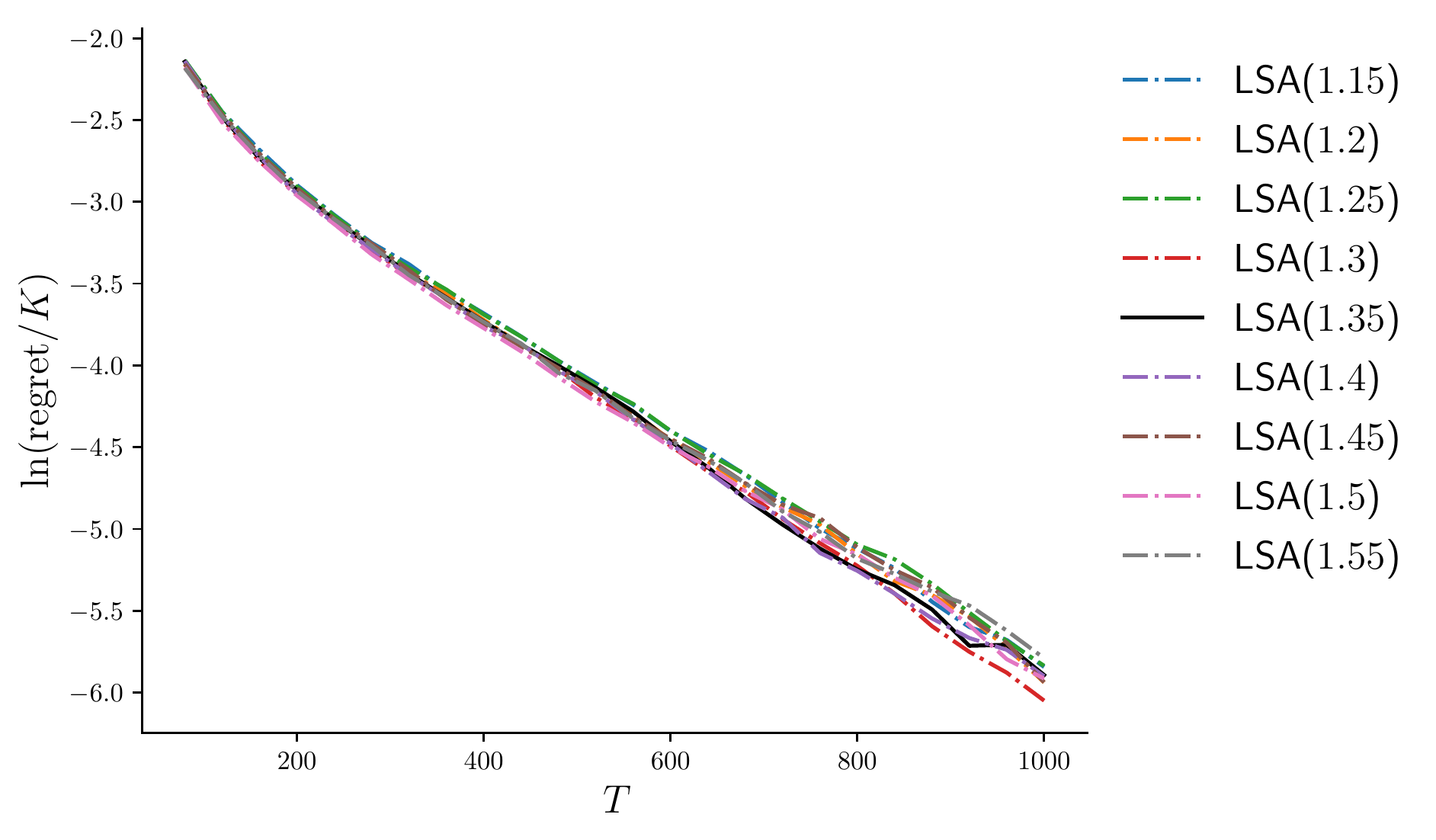}
\caption{Average aggregate regret on a logarithmic scale of $\LSA(\alpha)$ on Setup~1 for different $\alpha$.}
\label{fig:exp3}
\end{figure}

\end{document}